%% file: main.tex
\documentclass{article}

\PassOptionsToPackage{numbers, 
compress}{natbib}

\usepackage[preprint]{neurips_2026}

\usepackage[utf8]{inputenc} \usepackage[T1]{fontenc}    \usepackage{microtype}
\usepackage{graphicx}
\usepackage{subcaption}
\usepackage{booktabs} \usepackage{enumitem} \usepackage{wrapfig}  \usepackage{hyperref}
\usepackage{url}
\usepackage{float}
\usepackage{algorithm}
\usepackage{algorithmic}

\usepackage{amsmath}
\usepackage{amssymb}
\usepackage{amsfonts}
\usepackage{nicefrac}
\usepackage{mathtools}
\usepackage{amsthm}
\usepackage[capitalize,noabbrev]{cleveref}
\usepackage{listings}
\usepackage{xcolor}
\usepackage{bm}
\usepackage{bbm}

\definecolor{codegreen}{rgb}{0,0.6,0}
\definecolor{codegray}{rgb}{0.5,0.5,0.5}
\definecolor{codepurple}{rgb}{0.58,0,0.82}
\definecolor{backcolour}{rgb}{0.95,0.95,0.92}

\definecolor{softred}{RGB}{220, 80, 80}
\definecolor{softgreen}{RGB}{60, 150, 80}
\definecolor{softorange}{RGB}{230, 140, 60}
\definecolor{softblue}{RGB}{70, 130, 200}

\hypersetup{
    colorlinks=true,
    linkcolor=softblue,
    citecolor=softblue,
    urlcolor=softblue,
    filecolor=softblue
}

\lstdefinestyle{mystyle}{
    backgroundcolor=\color{backcolour},
    commentstyle=\color{codegreen},
    keywordstyle=\color{magenta},
    numberstyle=\tiny\color{codegray},
    stringstyle=\color{codepurple},
    basicstyle=\ttfamily\footnotesize,     breakatwhitespace=false,
    breaklines=true,
    captionpos=b,
    keepspaces=true,
    numbers=left,
    numbersep=5pt,
    showspaces=false,
    showstringspaces=false,
    showtabs=false,
    tabsize=4,
    language=Python
}

\theoremstyle{plain}
\newtheorem{theorem}{Theorem}[section]
\newtheorem{proposition}[theorem]{Proposition}

\theoremstyle{definition}

\theoremstyle{remark}

\usepackage[textsize=tiny]{todonotes}

\title{On Advantage Estimates for Max@K Policy Gradients}

\author{\makebox[\linewidth][c]{\hspace{-1.5em}
\begin{tabular}{@{}c@{}}
Shota Takashiro\textsuperscript{*} ~~ Soichiro Nishimori\textsuperscript{*} ~~  Paavo Parmas\textsuperscript{*} ~~  Yongmin Kim ~~  Kohsei Matsutani\\[0.3em]
 Gouki Minegishi ~~ Yusuke Iwasawa ~~  Takeshi Kojima ~~  Yutaka Matsuo
\\[0.5em]
\normalfont\footnotesize\textsuperscript{*}Equal contribution.\\[0.7em]
\normalfont{The University of Tokyo}\\
\texttt{\{takashiro,paavo.parmas\}@weblab.t.u-tokyo.ac.jp, nishimori@ms.k.u-tokyo.ac.jp}
\end{tabular}}
}

\begin{document}

\maketitle

\input{preamble.tex}

\newcommand{\fix}[1]{\textcolor{blue}{#1}}
\newcommand{\nishimori}[1]{\textcolor{blue}{[\textbf{Nishimori}]#1}}

\begin{abstract}
  Reinforcement learning with verifiable rewards is widely used for post-training reasoning models, but sparse outcome rewards make exploration difficult.
  A complementary approach is to optimize inference-time objectives such as pass@$K$ and max@$K$ directly, yet existing policy-gradient estimators for these objectives use different signals, baselines, and normalizations, making their relationships unclear.
  We study this issue through baseline design and advantage centering.
  Starting from the advantage estimator of a leading method in the field, we show that it is policy-gradient unbiased but yields a non-centered advantage.
  We then introduce a Leave-Two-Out baseline that preserves policy-gradient unbiasedness while making realized batch advantages exactly centered.
  The resulting method, MaxPO, has an efficient quadratic-time implementation and integrates naturally into group-based RL for LLM post-training.
  We further derive the canonical finite-batch advantage for max@$K$, providing a unified view of existing advantage estimators.
  Empirically, we verify that the L2O baseline reduces gradient variance and outperforms non-centered alternatives.
\end{abstract}

\input{sec/intro.tex}
\input{sec/background.tex}
\input{sec/method.tex}
\input{sec/unified.tex}
\input{sec/unified_experiment}
\input{sec/conclusion.tex}

\bibliography{ref}
\bibliographystyle{plainnat}

\newpage
\appendix
\input{sec/app/additional_related_work.tex}
\input{sec/app/code.tex}
\input{sec/app/algo.tex}
\input{sec/app/proof.tex}
\input{sec/app/computation.tex}
\input{sec/app/unified.tex}
\input{sec/app/toy_experiments.tex}
\input{sec/app/llm_experiments}

\section{LLM Usage}
LLMs were used to edit the manuscript and to assist with algebraic manipulations in derivations proposed by the authors. The key research ideas were proposed by the authors.

\end{document}

%% file: preamble.tex
\newcommand{\showfont}{encoding: \f@encoding{},
  family: \f@family{},
  series: \f@series{},
  shape: \f@shape{},
  size: \f@size{}
}

\urlstyle{same}

\def\eqcolon{=\mathrel{\mathop:}}
\newcommand{\argmax}{\operatornamewithlimits{argmax}}
\newcommand{\argmin}{\operatornamewithlimits{argmin}}

\newtheorem{example}{\textbf{例}}
\newtheorem*{example*}{\textbf{例}}
\numberwithin{example}{subsection}

\renewcommand*{\proofname}{\textbf{Proof}}

\newcommand{\R}{\mathbb{R}} \newcommand{\C}{\mathbb{C}} \newcommand{\Q}{\mathbb{Q}} \newcommand{\N}{\mathbb{N}} \newcommand{\Z}{\mathbb{Z}} \newcommand{\K}{\mathbb{K}} 
\newcommand{\sign}[1]{\mathrm{sign}\paren{#1}}

\newcommand{\abs}[1]{\left\lvert{#1}\right\rvert} \newcommand{\norm}[2][]{\lVert{#2}\rVert_{#1}} \newcommand{\Norm}[2][]{\left\lVert{#2}\right\rVert_{#1}} \newcommand{\supnorm}[1]{\norm[\infty]{#1}} \newcommand{\Supnorm}[1]{\Norm[\infty]{#1}} \newcommand{\enorm}[1]{\norm[2]{#1}} \newcommand{\opnorm}[1]{\norm[\mathrm{op}]{#1}} \newcommand{\fnorm}[1]{\norm[\mathrm{F}]{#1}} 
\newcommand{\psinorm}[2][]{\norm[\psi_{#1}]{#2}} \newcommand{\sgnorm}[1]{\psinorm[2]{#1}} \newcommand{\senorm}[1]{\psinorm[1]{#1}} \newcommand{\lipnorm}[1]{\norm[\mathrm{Lip}]{#1}}

\newcommand{\Enorm}[1]{\norm[2]{#1}} \newcommand{\Opnorm}[1]{\norm[\mathrm{op}]{#1}} \newcommand{\Fnorm}[1]{\norm[\mathrm{F}]{#1}} 
\newcommand{\Psinorm}[2][]{\Norm[\psi_{#1}]{#2}} \newcommand{\Sgnorm}[1]{\psiNorm[2]{#1}} \newcommand{\Senorm}[1]{\psiNorm[1]{#1}} \newcommand{\Lipnorm}[1]{\Norm[\mathrm{Lip}]{#1}}

\newcommand{\inv}[1]{{#1}^{-1}}
\newcommand{\Span}{\mathrm{span}}

\newcommand{\E}[2][]{\mathbb{E}_{#1} \left[ {#2} \right]} \renewcommand{\P}[2][]{P_{#1} \left[ {#2} \right]} \newcommand{\Var}[2][]{\mathrm{Var}_{#1} \left[ {#2} \right]} \newcommand{\Cov}[2][]{\mathrm{Cov}_{#1} \left[ {#2} \right]} \newcommand{\cov}[1]{\mathrm{cov}\paren{#1}} 
\newcommand{\Esub}[2]{\mathbb{E}_{#1} \left[ {#2} \right]}
\newcommand{\Varsub}[2]{\mathrm{Var}_{#1} \left[ {#2} \right]}

\newcommand{\B}{\mathcal{B}}
\newcommand{\D}{\mathcal{D}} \newcommand{\X}{\mathcal{X}} \newcommand{\Y}{\mathcal{Y}} \renewcommand{\H}{\mathcal{H}} \newcommand{\F}{\mathcal{F}}
\newcommand{\G}{\mathcal{G}}
\renewcommand{\L}{\mathcal{L}} \newcommand{\Rad}{\mathfrak{R}} \newcommand{\A}{\mathcal{A}} \renewcommand{\O}{\mathcal{O}} \newcommand{\I}{\mathcal{I}}

\newcommand{\ind}[1]{1_{#1}} \newcommand{\indic}[1]{\mathbbm{1}\left[{#1}\right]} 
\newcommand{\vol}{\mathrm{vol}}

\renewcommand{\epsilon}{\varepsilon}

\newcommand{\cA}{\mathcal{A}}
\newcommand{\cB}{\mathcal{B}}
\newcommand{\cC}{\mathcal{C}}
\newcommand{\cD}{\mathcal{D}}
\newcommand{\cE}{\mathcal{E}}
\newcommand{\cF}{\mathcal{F}}
\newcommand{\cG}{\mathcal{G}}
\newcommand{\cH}{\mathcal{H}}
\newcommand{\cI}{\mathcal{I}}
\newcommand{\cJ}{\mathcal{J}}
\newcommand{\cK}{\mathcal{K}}
\newcommand{\cL}{\mathcal{L}} \newcommand{\cM}{\mathcal{M}}
\newcommand{\cN}{\mathcal{N}}
\newcommand{\cO}{\mathcal{O}} \newcommand{\cP}{\mathcal{P}}
\newcommand{\cQ}{\mathcal{Q}}
\newcommand{\cR}{\mathcal{R}}
\newcommand{\cS}{\mathcal{S}}
\newcommand{\cT}{\mathcal{T}}
\newcommand{\cU}{\mathcal{U}}
\newcommand{\cX}{\mathcal{X}}
\newcommand{\cY}{\mathcal{Y}}
\newcommand{\cZ}{\mathcal{Z}}

\newcommand{\fA}{\mathfrak{A}}
\newcommand{\fB}{\mathfrak{B}}
\newcommand{\fC}{\mathfrak{C}}
\newcommand{\fD}{\mathfrak{D}}
\newcommand{\fE}{\mathfrak{E}}
\newcommand{\fF}{\mathfrak{F}}
\newcommand{\fG}{\mathfrak{G}}
\newcommand{\fH}{\mathfrak{H}}
\newcommand{\fI}{\mathfrak{I}}
\newcommand{\fJ}{\mathfrak{J}}
\newcommand{\fK}{\mathfrak{K}}
\newcommand{\fL}{\mathfrak{L}}
\newcommand{\fM}{\mathfrak{M}}
\newcommand{\fN}{\mathfrak{N}}
\newcommand{\fO}{\mathfrak{O}}
\newcommand{\fP}{\mathfrak{P}}
\newcommand{\fQ}{\mathfrak{Q}}
\newcommand{\fR}{\mathfrak{R}}
\newcommand{\fS}{\mathfrak{S}}
\newcommand{\fT}{\mathfrak{T}}
\newcommand{\fU}{\mathfrak{U}}
\newcommand{\fV}{\mathfrak{V}}
\newcommand{\fW}{\mathfrak{W}}
\newcommand{\fX}{\mathfrak{X}}
\newcommand{\fY}{\mathfrak{Y}}
\newcommand{\fZ}{\mathfrak{Z}}

\newcommand{\sA}{\mathscr{A}}
\newcommand{\sB}{\mathscr{B}}
\newcommand{\sC}{\mathscr{C}}
\newcommand{\sD}{\mathscr{D}}
\newcommand{\sE}{\mathscr{E}}
\newcommand{\sF}{\mathscr{F}}
\newcommand{\sG}{\mathscr{G}}
\newcommand{\sH}{\mathscr{H}}
\newcommand{\sI}{\mathscr{I}}
\newcommand{\sJ}{\mathscr{J}}
\newcommand{\sK}{\mathscr{K}}
\newcommand{\sL}{\mathscr{L}}
\newcommand{\sM}{\mathscr{M}}
\newcommand{\sN}{\mathscr{N}}
\newcommand{\sO}{\mathscr{O}}
\newcommand{\sP}{\mathscr{P}}
\newcommand{\sQ}{\mathscr{Q}}
\newcommand{\sR}{\mathscr{R}}
\newcommand{\sS}{\mathscr{S}}
\newcommand{\sT}{\mathscr{T}}
\newcommand{\sU}{\mathscr{U}}
\newcommand{\sV}{\mathscr{V}}
\newcommand{\sW}{\mathscr{W}}
\newcommand{\sX}{\mathscr{X}}
\newcommand{\sY}{\mathscr{Y}}
\newcommand{\sZ}{\mathscr{Z}}

\newcommand{\rA}{\mathrm{A}}
\newcommand{\rB}{\mathrm{B}}
\newcommand{\rC}{\mathrm{C}}
\newcommand{\rD}{\mathrm{D}}
\newcommand{\rE}{\mathrm{E}}
\newcommand{\rF}{\mathrm{F}}
\newcommand{\rG}{\mathrm{G}}
\newcommand{\rH}{\mathrm{H}}
\newcommand{\rI}{\mathrm{I}}
\newcommand{\rJ}{\mathrm{J}}
\newcommand{\rK}{\mathrm{K}}
\newcommand{\rL}{\mathrm{L}}
\newcommand{\rM}{\mathrm{M}}
\newcommand{\rN}{\mathrm{N}}
\newcommand{\rO}{\mathrm{O}}
\newcommand{\rP}{\mathrm{P}}
\newcommand{\rQ}{\mathrm{Q}}
\newcommand{\rR}{\mathrm{R}}
\newcommand{\rS}{\mathrm{S}}
\newcommand{\rT}{\mathrm{T}}
\newcommand{\rU}{\mathrm{U}}
\newcommand{\rV}{\mathrm{V}}
\newcommand{\rW}{\mathrm{W}}
\newcommand{\rX}{\mathrm{X}}
\newcommand{\rY}{\mathrm{Y}}
\newcommand{\rZ}{\mathrm{Z}}

\newcommand{\ra}{\mathrm{a}}
\newcommand{\rb}{\mathrm{b}}
\newcommand{\rc}{\mathrm{c}}
\newcommand{\rd}{\mathrm{d}}
\newcommand{\re}{\mathrm{e}}
\newcommand{\rf}{\mathrm{f}}
\newcommand{\rg}{\mathrm{g}}
\newcommand{\rh}{\mathrm{h}}
\newcommand{\ri}{\mathrm{i}}
\newcommand{\rj}{\mathrm{j}}
\newcommand{\rk}{\mathrm{k}}
\newcommand{\rl}{\mathrm{l}}
\newcommand{\rn}{\mathrm{n}}
\newcommand{\ro}{\mathrm{o}}
\newcommand{\rp}{\mathrm{p}}
\newcommand{\rr}{\mathrm{r}}
\newcommand{\rs}{\mathrm{s}}
\newcommand{\rt}{\mathrm{t}}
\newcommand{\ru}{\mathrm{u}}
\newcommand{\rv}{\mathrm{v}}
\newcommand{\rw}{\mathrm{w}}
\newcommand{\rx}{\mathrm{x}}
\newcommand{\ry}{\mathrm{y}}
\newcommand{\rz}{\mathrm{z}}

\newcommand{\bone}{\mathbf{1}}
\newcommand{\bc}{\mathbf{c}}
\newcommand{\bx}{\mathbf{x}}
\newcommand{\bw}{\mathbf{w}}
\newcommand{\bK}{\mathbf{K}}

\newcommand{\bmA}{\bm{A}}
\newcommand{\bmB}{\bm{B}}
\newcommand{\bmC}{\bm{C}}
\newcommand{\bmD}{\bm{D}}
\newcommand{\bmE}{\bm{E}}
\newcommand{\bmF}{\bm{F}}
\newcommand{\bmG}{\bm{G}}
\newcommand{\bmH}{\bm{H}}
\newcommand{\bmI}{\bm{I}}
\newcommand{\bmJ}{\bm{J}}
\newcommand{\bmK}{\bm{K}}
\newcommand{\bmL}{\bm{L}}
\newcommand{\bmM}{\bm{M}}
\newcommand{\bmN}{\bm{N}}
\newcommand{\bmO}{\bm{O}}
\newcommand{\bmP}{\bm{P}}
\newcommand{\bmQ}{\bm{Q}}
\newcommand{\bmR}{\bm{R}}
\newcommand{\bmS}{\bm{S}}
\newcommand{\bmT}{\bm{T}}
\newcommand{\bmU}{\bm{U}}
\newcommand{\bmV}{\bm{V}}
\newcommand{\bmW}{\bm{W}}
\newcommand{\bmX}{\bm{X}}
\newcommand{\bmY}{\bm{Y}}
\newcommand{\bmZ}{\bm{Z}}
\newcommand{\bma}{\bm{a}}
\newcommand{\bmb}{\bm{b}}
\newcommand{\bmc}{\bm{c}}
\newcommand{\bmd}{\bm{d}}
\newcommand{\bme}{\bm{e}}
\newcommand{\bmf}{\bm{f}}
\newcommand{\bmg}{\bm{g}}
\newcommand{\bmh}{\bm{h}}
\newcommand{\bmi}{\bm{i}}
\newcommand{\bmj}{\bm{j}}
\newcommand{\bmk}{\bm{k}}
\newcommand{\bml}{\bm{l}}
\newcommand{\bmm}{\bm{m}}
\newcommand{\bmn}{\bm{n}}
\newcommand{\bmo}{\bm{o}}
\newcommand{\bmp}{\bm{p}}
\newcommand{\bmq}{\bm{q}}
\newcommand{\bmr}{\bm{r}}
\newcommand{\bms}{\bm{s}}
\newcommand{\bmt}{\bm{t}}
\newcommand{\bmu}{\bm{u}}
\newcommand{\bmv}{\bm{v}}
\newcommand{\bmw}{\bm{w}}
\newcommand{\bmx}{\bm{x}}
\newcommand{\bmy}{\bm{y}}
\newcommand{\bmz}{\bm{z}}
\newcommand{\bmsigma}{\bm{\sigma}}
\newcommand{\bmmu}{\bm{\mu}}

\newcommand{\bfA}{\mathbf{A}}
\newcommand{\bfB}{\mathbf{B}}
\newcommand{\bfC}{\mathbf{C}}
\newcommand{\bfD}{\mathbf{D}}
\newcommand{\bfE}{\mathbf{E}}
\newcommand{\bfF}{\mathbf{F}}
\newcommand{\bfG}{\mathbf{G}}
\newcommand{\bfH}{\mathbf{H}}
\newcommand{\bfI}{\mathbf{I}}
\newcommand{\bfJ}{\mathbf{J}}
\newcommand{\bfK}{\mathbf{K}}
\newcommand{\bfL}{\mathbf{L}}
\newcommand{\bfM}{\mathbf{M}}
\newcommand{\bfN}{\mathbf{N}}
\newcommand{\bfO}{\mathbf{O}}
\newcommand{\bfP}{\mathbf{P}}
\newcommand{\bfQ}{\mathbf{Q}}
\newcommand{\bfR}{\mathbf{R}}
\newcommand{\bfS}{\mathbf{S}}
\newcommand{\bfT}{\mathbf{T}}
\newcommand{\bfU}{\mathbf{U}}
\newcommand{\bfV}{\mathbf{V}}
\newcommand{\bfW}{\mathbf{W}}
\newcommand{\bfX}{\mathbf{X}}
\newcommand{\bfY}{\mathbf{Y}}
\newcommand{\bfZ}{\mathbf{Z}}
\newcommand{\bfa}{\mathbf{a}}
\newcommand{\bfb}{\mathbf{b}}
\newcommand{\bfc}{\mathbf{c}}
\newcommand{\bfd}{\mathbf{d}}
\newcommand{\bfe}{\mathbf{e}}
\newcommand{\bff}{\mathbf{f}}
\newcommand{\bfg}{\mathbf{g}}
\newcommand{\bfh}{\mathbf{h}}
\newcommand{\bfi}{\mathbf{i}}
\newcommand{\bfj}{\mathbf{j}}
\newcommand{\bfk}{\mathbf{k}}
\newcommand{\bfl}{\mathbf{l}}
\newcommand{\bfm}{\mathbf{m}}
\newcommand{\bfn}{\mathbf{n}}
\newcommand{\bfo}{\mathbf{o}}
\newcommand{\bfp}{\mathbf{p}}
\newcommand{\bfq}{\mathbf{q}}
\newcommand{\bfr}{\mathbf{r}}
\newcommand{\bfs}{\mathbf{s}}
\newcommand{\bft}{\mathbf{t}}
\newcommand{\bfu}{\mathbf{u}}
\newcommand{\bfv}{\mathbf{v}}
\newcommand{\bfw}{\mathbf{w}}
\newcommand{\bfx}{\mathbf{x}}
\newcommand{\bfy}{\mathbf{y}}
\newcommand{\bfz}{\mathbf{z}}

\newcommand{\bbA}{\mathbb{A}}
\newcommand{\bbB}{\mathbb{B}}
\newcommand{\bbC}{\mathbb{C}}
\newcommand{\bbD}{\mathbb{D}}
\newcommand{\bbE}{\mathbb{E}}
\newcommand{\bbF}{\mathbb{F}}
\newcommand{\bbG}{\mathbb{G}}
\newcommand{\bbH}{\mathbb{H}}
\newcommand{\bbI}{\mathbb{I}}
\newcommand{\bbJ}{\mathbb{J}}
\newcommand{\bbK}{\mathbb{K}}
\newcommand{\bbL}{\mathbb{L}}
\newcommand{\bbM}{\mathbb{M}}
\newcommand{\bbN}{\mathbb{N}}
\newcommand{\bbO}{\mathbb{O}}
\newcommand{\bbP}{P}
\newcommand{\bbQ}{\mathbb{Q}}
\newcommand{\bbR}{\mathbb{R}}
\newcommand{\bbS}{\mathbb{S}}
\newcommand{\bbT}{\mathbb{T}}
\newcommand{\bbU}{\mathbb{U}}
\newcommand{\bbV}{\mathbb{V}}
\newcommand{\bbW}{\mathbb{W}}
\newcommand{\bbX}{\mathbb{X}}
\newcommand{\bbY}{\mathbb{Y}}
\newcommand{\bbZ}{\mathbb{Z}}
\newcommand{\eR}{\overline{\R}}
\newcommand{\eO}{\overline{\cO}}

\newcommand{\st}{\,\mathrm{s.t.}\,}
\newcommand{\iid}{\mathrm{i.i.d.}}

\newcommand{\set}[1]{\left\{{#1}\right\}}
\newcommand{\seq}[3][]{\set{#2}_{#3}^{#1}}
\newcommand{\infseq}[3][1]{\seq[\infty]{#2}{{#3}={#1}}}
\newcommand{\ninfseq}[2][1]{\seq[\infty]{#2}{n={#1}}}

\newcommand{\infbigcup}[2][1]{\bigcup_{{#2}={#1}}^\infty}
\newcommand{\ninfbigcup}[1][1]{\infbigcup[#1]{n}}

\newcommand{\comp}[1]{\left.{#1}\right.^\mathrm{c}}
\newcommand{\interior}[1]{\left.{#1}\right.^\mathrm{\circ}}
\newcommand{\exterior}[1]{\left.{#1}\right.^\mathrm{e}}
\newcommand{\boundary}[1]{\mathrm{bd}\paren{#1}}

\newcommand{\pcomp}[1]{\left({#1}\right)^\mathrm{c}}
\newcommand{\pinterior}[1]{\left({#1}\right)^\mathrm{\circ}}
\newcommand{\pexterior}[1]{\left({#1}\right)^\mathrm{e}}
\newcommand{\pboundary}[1]{\left({#1}\right)^\mathrm{f}}

\newcommand{\relint}[1]{\mathrm{relint}\paren{#1}}
\newcommand{\inter}[1]{\mathrm{int}\paren{#1}}
\newcommand{\conv}[1]{\mathrm{conv}\paren{#1}}
\newcommand{\aff}[1]{\mathrm{aff}\paren{#1}}
\newcommand{\cone}[1]{\mathrm{cone}\paren{#1}}

%% file: sec/intro.tex
\section{Introduction}
Reinforcement learning with verifiable rewards has become a common post-training method for large language models, particularly for reasoning tasks where the final answer can be checked automatically~\citep{shao2024deepseekmath}.
In such settings, rewards are typically sparse and outcome-based: a sampled completion receives a positive reward only if its final answer is correct.
This sparsity makes exploration difficult, since the policy must discover useful reasoning paths in a large combinatorial space.
Existing work has addressed this issue by adding exploration incentives, including entropy regularization~\citep{cheng2025reasoning, zheng2025first, cui2025entropy} and count-based bonuses~\citep{song2025outcome, bai2025online}.
These methods can be effective, but they introduce auxiliary objectives and require careful tuning of their trade-off with the task reward.

A complementary direction is to optimize the objective that is actually used at inference time.
Reasoning models are often evaluated by repeated sampling: a prompt is considered solved if at least one of several sampled completions is correct.
This motivates pass@$K$, the probability that at least one out of $K$ samples is correct, and its continuous-reward generalization max@$K$.
Several recent works have proposed unbiased policy-gradient estimators for pass@$K$ and max@$K$~\citep{koyamada2023emergence,tang2025optimizing, zheng2025act, chen2025pass, walder2025pass, peng2025simko, bagirov2025best, tuyls2025representation}.
However, because the max operation couples multiple samples in a nontrivial way, the relationship among these estimators is not immediately clear.
As a result, it remains unclear which estimator should be preferred, even when they target the same objective.

We study this question through a classical criterion from the policy-gradient literature: the quality of the advantage estimator.
Subtracting a baseline can reduce variance while preserving unbiasedness~\citep{Williams1992-rp, greensmith2004variance, mei2022role}, and a particularly desirable advantage estimator is centered.
For batched policy-gradient estimators, an even stronger property is exact centering: the realized batch-average advantage is zero for every sampled batch.
This property removes unnecessary common-mode variation in the gradient estimate and provides a principled way to compare otherwise unbiased estimators.

From this perspective, we first revisit the estimator of \citet{walder2025pass} from the
Expected Improvement (EI) viewpoint from our prior work \citep{nishimori2026emergence}.
This estimator provides a general unbiased policy-gradient estimator for max@$K$ with $K \le B$, where $B$ is the number of sampled responses.
We show that the PG estimator is unbiased but the resulting advantage is not centered, leaving room for variance reduction through an additional baseline.
To address this, we introduce a \textbf{Leave-Two-Out (L2O)} baseline.
The resulting estimator remains unbiased as a policy-gradient estimator, while its advantage is zero-mean in expectation and exactly centered within every realized batch.
We also derive an efficient $O(B^2)$ vectorized implementation and incorporate the estimator into group-based reinforcement learning for LLM post-training, which we call MaxPO, or Max@$K$ Policy Optimization.

We then step back from the EI estimator and ask what the advantage should be for a finite-batch max@$K$ policy gradient.
In ordinary policy gradients, an advantage is a return minus a baseline.
For max@$K$, however, the relevant return of a response is not its individual reward alone, because the objective evaluates the maximum reward within a group of $K$ samples.
Thus, the natural return assigned to response $i$ is the expected max@$K$ value of a group conditioned on containing $i$, which we denote by $u_i$.
The corresponding baseline should remove the contribution of response $i$ while preserving the same max@$K$ objective; this gives the leave-one-out expected max@$K$ value $v_i$, computed from groups that exclude $i$.
This leads to the canonical finite-batch advantage $u_i-v_i$: the max@$K$ analogue of a return minus a leave-one-out baseline.
This view explains why MaxPO recovers the centered canonical signal, and it also provides a common language for comparing existing estimators.
Under this view, current methods fall into two categories: \textbf{uncentered} signals such as EI-only/PKPO, and \textbf{(normalized) canonical} signals such as MaxPO.

Empirically, we first validate the variance-reduction effect of our method in controlled bandit and maze environments, with the maze results reported in Appendix~\ref{app:toy_experiments:maze_environment}.
We also validate the practical efficacy of our method for LLM reasoning tasks using Llama-3.2-3B-Instruct and Qwen2.5-Math-7B (Sec.~\ref{sec:reasoning_experiments}).
Specifically, on the Llama-3.2-3B-Instruct model, our method (EI+L2O) achieved an average reduction of 77.4\% in gradient variance during training compared to PKPO, an existing method without L2O \citep{walder2025pass}.
Our method also improved task-average pass@256 performance by 5.2\% on Qwen2.5-Math-7B and by 2.4\% on Llama-3.2-3B-Instruct relative to PKPO across five math reasoning benchmarks (AIME24, AIME25, AMC23, MATH500~\citep{hendrycks2021math}, and Minerva~\citep{lewkowycz2022solving}).
Furthermore, we demonstrated that our method consistently outperforms strong representative LLM post-training baselines (PKPO, GRPO, Entropy-Adv).

%% file: sec/background.tex
\section{Background}\label{sec:background}
In this section, we first formalize the setting of reasoning tasks and review the objectives of pass@\textit{K} and max@\textit{K} policy optimization.
We then outline the motivation for our work and introduce the fundamental concepts required to construct our proposed estimator.

\paragraph{Setting.}
\label{sec:background:setting}
We consider a setting in which an agent generates an action $a \in \mathcal{A}$, where $\mathcal{A}$ is a finite set, and the action is evaluated by a reward function $r: \mathcal{A} \to \mathbb{R}$.
Our goal is to learn a policy $\pi_\theta \in \Delta(\mathcal{A})$, parameterized by $\theta \in \mathbb{R}^d$, that maximizes the expected reward
$J_{\mathrm{RL}}(\theta) = \mathbb{E}_{a \sim \pi_\theta}[r(a)]$.
When it is clear from the context, we omit the dependence on $\pi_\theta$ for brevity.
In practice, the policy may take additional information as input (e.g., questions in LLM reasoning), but we omit this dependence, and the same statements naturally apply to that setting (Sec.~\ref{sec:method:integration}).
During optimization, we assume access to a batch of $B$ actions $a_1,\ldots,a_B \in \mathcal{A}$ sampled from the policy and rewards $r_1,\ldots,r_B \in \mathbb{R}$.
We denote the batch by $\mathcal{D}=(a_{1:B},r_{1:B})$.

\subsection{Policy Gradient Estimation and Variance Reduction}\label{sec:background:pg}
In this study, we focus on the policy gradient (PG) method, which directly optimizes $\pi_\theta$ via gradient ascent using $\nabla J(\theta) = \mathbb{E}_{a \sim \pi_\theta}\!\left[r(a)\nabla_{\theta}\log \pi_\theta(a)\right]$ \citep{Williams1992-rp}.
Because policy gradient estimates can exhibit high variance, we introduce a constant baseline $b$ to reduce variance \citep{greensmith2004variance}.
This yields the following expression.
\begin{equation}\label{eq:policy_gradient}
\nabla J(\theta)
= \mathbb{E}_{a \sim \pi_\theta}
\left[(r(a) - b)\nabla_{\theta}\log \pi_\theta(a)\right].
\end{equation}
We use the term \emph{advantage} to denote the action-dependent scalar multiplier of the score function (i.e., \mbox{$r(a)-b$} in Eq.~\eqref{eq:policy_gradient}).
As long as the baseline is action-independent, the PG estimator remains unbiased because $\mathbb{E}_{a \sim \pi_\theta}\!\left[b\,\nabla_\theta \log \pi_\theta(a)\right] = 0$.
A common choice is to use the expected reward as the baseline, $b = \mathbb{E}_{a \sim \pi_\theta}[r(a)]$, so that the advantage has zero mean: \mbox{$\mathbb{E}_{a \sim \pi_\theta}[r(a)-b]=0$} \citep{greensmith2004variance}.
Using such a zero-mean advantage removes a constant offset from the reward signal, which is often a dominant source of variance.

\paragraph{Batch estimation.}
Given the batch $\mathcal{D}$, we can approximate the gradient using the REINFORCE estimator \citep{Williams1992-rp} $\hat{g} = \frac{1}{B} \sum_{i=1}^B (r_i - b_i)\nabla_{\theta} \log \pi_\theta(a_i).$
Here, $b_i$ is the baseline for sample $i$.
We distinguish two notions of unbiasedness and one notion of centering:
\begin{enumerate}
\item \textbf{Unbiased PG estimator:} We call a PG estimator $\hat{g}$ unbiased when $\mathbb{E}_{\mathcal{D} \sim \pi_\theta}\!\left[\hat{g}\right] = \nabla J(\theta)$.
Here, $\mathbb{E}_{\mathcal{D} \sim \pi_\theta}[\cdot]$ denotes expectation over random batches $\mathcal{D}=(a_{1:B}, r_{1:B})$ obtained by i.i.d. sampling from $\pi_\theta$ and evaluating rewards under $r$.
This requires that the baseline $b_i$ does not depend on $r_i$, which would introduce action dependence into the PG estimator.
\item \textbf{Unbiased advantage:} We call an advantage estimator (e.g., \mbox{$r_i-b_i$}) unbiased when \mbox{$\mathbb{E}_{\mathcal{D} \sim \pi_\theta}[r_i-b_i]=0$}.
\item \textbf{Centering:} We call the realized batch-average advantage centered when \mbox{$\frac{1}{B}\sum_{i=1}^B (r_i-b_i)=0$}.
This is a stronger per-batch property than unbiasedness in expectation; the latter does not in general imply centering.
\end{enumerate}
Unbiasedness of the PG estimator is a prerequisite for policy gradient methods, and an unbiased advantage is crucial for variance reduction \citep{greensmith2004variance}.
Therefore, we aim to achieve both forms of unbiasedness, and it turns out that centering also holds.
For a batch PG estimator, a na\"ive way to obtain an unbiased advantage estimator is to set the baseline to the batch mean reward, $b_i = \frac{1}{B}\sum_{j=1}^B r_j$, as an estimator of $\mathbb{E}_{a \sim \pi_\theta}[r(a)]$.
However, this violates the independence condition because $b_i$ depends on $r_i$, resulting in a biased PG estimator.
A common remedy is to use a Leave-One-Out (L1O) baseline, $b_i = \frac{1}{B-1}\sum_{j \neq i} r_j$, to avoid dependence on $r_i$ \citep{pmlr-v80-parmas18a,mnih2016asynchronous}.
Although L1O works in this simple RL setting, it may fail for more complicated objectives, as is the case for max@K (see Sec.~\ref{sec:method:ei}).

\subsection{Pass@K and max@K Policy Optimization}\label{sec:background:pass@K}
Recently, in the context of RL for reasoning tasks with binary rewards (correct/incorrect), directly optimizing the pass@\textit{K} metric has been proposed to incentivize answer diversity \citep{tang2025optimizing, zheng2025act, chen2025pass, walder2025pass, peng2025simko, bagirov2025best, tuyls2025representation}.
Pass@\textit{K} optimization aims to maximize the probability that at least one of $K$ sampled answers is correct.
The max@\textit{K} objective generalizes pass@\textit{K} to continuous reward functions \citep{walder2025pass} and is defined as the expected maximum reward among $K \leq B$ samples.
\begin{equation}\label{eq:max@K}
J^{K}(\theta) := \mathbb{E}_{a_{1:K} \sim \pi_\theta}\!\left[\max_{k=1,\dots,K} r(a_k)\right].
\end{equation}

The PG for this objective was first given in our prior work with \citet{koyamada2023emergence}:
\begin{equation}\label{eq:policy_gradient_max@K}
\nabla J^{K}(\theta)
:= \mathbb{E}_{a_{1:K}}
\left[\max_{k=1,\dots,K} r(a_k) \sum_{k=1}^K \nabla_{\theta} \log \pi_\theta(a_k)\right].
\end{equation}

Our goal is to propose an estimator that is unbiased as a PG estimator and also yields an advantage with zero expectation.
However, the formulation in Eq.~\eqref{eq:policy_gradient_max@K} makes it difficult to evaluate the advantage because the $K$ actions are coupled through the max operator.
To address this issue, we leverage an alternative formulation of the max@\textit{K} PG.

\paragraph{Expected Improvement Formulation.}
\citet{nishimori2026emergence} proposed a method to decouple the max@\textit{K} PG into a per-action expectation by building on the baseline proposed by \citet{tang2025optimizing}:
$W_{-k} := \max_{k' \neq k} r(a_{k'})$.
Using the identity $\max_{k} r_k - W_{-k} = (r_k - W_{-k})_+$, where $(z)_+ := \max(z, 0)$, the PG can be rewritten as:
\begin{align*}
\nabla J^{K}(\theta)&= \mathbb{E}\left[ \sum_{k=1}^K \nabla_{\theta} \log \pi_\theta(a_k)\, \bigl(r(a_k) - W_{-k}\bigr)_+ \right].
\end{align*}

By exploiting the symmetry under the i.i.d.\ assumption, we arrive at the simplified form (Prop.~2 of \citet{nishimori2026emergence}):
\begin{equation}\label{eq:pg_ei}
\nabla J^{K}(\theta) = K \mathbb{E}_{a \sim \pi_\theta}
\left[\nabla_{\theta} \log \pi_\theta(a)\, s(a) \right], \quad s(a) := \mathbb{E}_{a_{1:K-1} \sim \pi_\theta}
\left[ \bigl(r(a) - W_{K-1}\bigr)_+ \right],
\end{equation}
with $W_{K-1} = \max_{j=1,\dots,K-1} r(a_j)$.
$s(a)$ is referred to as the \textbf{Expected Improvement (EI)} of action $a$, as it quantifies the expected gain of $a$ over the best of $K-1$ other samples \citep{nishimori2026emergence}.

\paragraph{Motivation for Further Variance Reduction.}
Compared to the standard PG in RL (Eq.~\eqref{eq:policy_gradient}), the EI term $s(a)$ serves as the advantage.
Moreover, $s(a)$ corresponds to the expected value of the advantage in the $\max@K - \max@(K-1)$ estimator proposed by \citet{walder2025pass}.
While their estimator reduces variance compared to na\"ive approaches, Eq.~\eqref{eq:pg_ei} reveals a critical limitation.
Because $s(a)$ is the output of a ReLU, it is non-negative, implying $\mathbb{E}_{a \sim \pi_\theta}[s(a)] > 0$.
This consistently positive advantage corresponds to an underestimated baseline, which can be particularly harmful because it may steer policy optimization toward suboptimal actions \citep{chung2021beyond}.
Thus, \textbf{their estimator is unbiased as a PG estimator but not as an advantage estimator}, leaving room for further variance reduction through appropriate baseline subtraction.

%% file: sec/method.tex
\section{Proposed Method}\label{sec:method}
In Sec.~\ref{sec:method:ei}, we derive a baseline that \textbf{guarantees unbiasedness of both the PG estimator and the resulting advantage and exactly centers the realized batch-average advantage}, and we present an efficient computation method.
Finally, Sec.~\ref{sec:method:integration} describes how to integrate our approach into modern RL algorithms, instantiating it as MaxPO (Max@K Policy Optimization).

\subsection{Unbiased Advantage Estimation}\label{sec:method:ei}
Given a batch $\mathcal{D} = (a_{1:B}, r_{1:B})$, we first construct unbiased estimators $s_i$ of the expected improvement $s(a_i)$ defined in Eq.~\eqref{eq:pg_ei}.
\begin{equation}
s_i := \mathbb{E}_{\mathcal{I}}\!\left[\bigl(r_i - \max_{j \in \mathcal{I}} r_j\bigr)_+\right].
\end{equation}
Here, $\mathcal{I}$ is a subset of size $K-1$ drawn uniformly \textit{without replacement} from the indices $\{1, \dots, B\}\setminus\{i\}$.
By U-statistics theory \citep{hoeffding1992class}, $s_i$ is an unbiased estimator of $s(a_i)$.
Note that $s_i$ is equivalent to one of \citet{walder2025pass}'s advantage estimators (max@K - max@(K-1)).

Recall the three notions introduced in Sec.~\ref{sec:background:pg}.
For the desired estimator, we seek a baseline $b_i$ that satisfies the following three conditions:
(1) \textbf{Independence:} $b_i$ does not depend on $r_i$, ensuring an unbiased PG estimator;
(2) \textbf{Unbiased advantage:} \mbox{$\mathbb{E}_{\mathcal{D} \sim \pi_\theta}[s_i-b_i]=0$}; and
(3) \textbf{Centering:} \mbox{$\frac{1}{B}\sum_{i=1}^B (s_i-b_i)=0$} for every realized batch.
As discussed in Sec.~\ref{sec:background:pg}, a common approach is the L1O baseline, $b_{-i}^{\mathrm{L1O}} := \frac{1}{B-1} \sum_{j \neq i} s_j$.
However, in the max@K setting, the estimator $s_j, j \neq i$ can implicitly depend on $r_i$, because $r_i$ acts as a potential comparator (maximum over the subset of size $K-1$) in the construction of $s_j$.
Consequently, $b_{-i}^{\mathrm{L1O}}$ can correlate with $r_i$, thereby violating the independence condition and leading to a biased PG estimator.

\paragraph{Leave-Two-Out (L2O) Baseline.}
To satisfy all three conditions, we propose the \textit{Leave-Two-Out} (L2O) baseline:
\begin{equation}\label{eq:l2o_baseline}
b_{-i}^{\mathrm{L2O}} := \frac{1}{B-1} \sum_{j \neq i} s^{(-i)}_j, \quad s^{(-i)}_j := \mathbb{E}_{\mathcal{I}'}\!\left[ \bigl(r_j - \max_{k \in \mathcal{I}'} r_k\bigr)_+ \right].
\end{equation}
Here, $\mathcal{I}'$ is a subset of size $K-1$ drawn uniformly from $\{1, \dots, B\}\setminus\{i, j\}$.
We exclude $i$ to ensure independence from $r_i$, while keeping the comparator subset size $K-1$ consistent with the original EI definition.
This requires $K - 1 \leq B - 2$, i.e., $K \leq B - 1$.
By construction, $b_{-i}^{\mathrm{L2O}}$ is independent of $r_i$ because $r_i$ does not appear as a comparator in its computation.
The following proposition establishes the centering condition. Together with the independence property above, it implies that the full estimator $\frac{K}{B} \sum_{i=1}^B \nabla_{\theta} \log \pi_\theta(a_i) (s_i - b_{-i}^{\mathrm{L2O}})$ is an unbiased PG estimator for the max@\textit{K} objective, and that $s_i-b_{-i}^{\mathrm{L2O}}$ is an unbiased advantage estimator.

\begin{proposition}[Unbiasedness and Centering of the L2O Baseline]\label{prop:l2o_baseline_unbiasedness}
For any $i$, the L2O advantage satisfies
\begin{equation}
    \mathbb{E}_{\mathcal{D} \sim \pi_\theta}\!\left[s_i - b_{-i}^{\mathrm{L2O}}\right] = 0, \quad \text{and} \quad     \frac{1}{B}\sum_{i=1}^B \left(s_i - b_{-i}^{\mathrm{L2O}}\right) = 0,
\end{equation}
for any realized batch $\mathcal{D}$.
Hence $s_i - b_{-i}^{\mathrm{L2O}}$ is unbiased as an advantage estimator and exactly centered at the batch level.
\end{proposition}
The proof is in Appendix~\ref{app:proof:l2o_baseline_unbiasedness}.
While the theoretical properties of the L2O baseline are desirable, its practical application hinges on computational efficiency.
Computing $\mathbf{b}^{\mathrm{L2O}} := [b_{-1}^{\mathrm{L2O}}, \dots, b_{-B}^{\mathrm{L2O}}]^\top$ naively costs $O(B^3)$ because it evaluates $O(B)$ EI terms for each of the $B$ indices $i$, and each EI term naively scans $O(B)$ candidate comparators in the (reduced) batch.
However, leveraging the ReLU-based formulation of the EI estimator, we can compute the L2O baseline in $O(B^2)$.

\begin{theorem}[Efficient Computation of the L2O Baseline]\label{thm:l2o_baseline_informal}
    Given a batch of rewards $\mathbf{r} = [r_1, \dots, r_B]^\top$, the L2O baseline $\mathbf{b}^{\mathrm{L2O}}$ can be computed in $O(B^2)$ time.
\end{theorem}
The proof is provided in Appendix~\ref{app:computation}.
The resulting vectorized formulation maps directly to GPU-friendly dense primitives (e.g., broadcasted elementwise operations and matrix--vector products), making it straightforward to leverage GPU parallelism in practice.

\begin{algorithm}[t]
    \caption{Max@K Policy Optimization}
    \label{alg:maxpo}
    \begin{algorithmic}[1]
    \STATE \textbf{Input:} policy $\pi_\theta$, objective size $K$, number of questions per batch $M$, group size $G$ with $2 \le K \le G-1$.
    \WHILE{not converged}
        \STATE Sample questions $\{x_m\}_{m=1}^M$.
        \FOR{each question $x_m$}
            \STATE Sample $G$ outputs $\{a_j^m\}_{j=1}^G \sim \pi_\theta(\cdot \mid x_m)$ and evaluate rewards $\{r_j^m\}_{j=1}^G$.
            \STATE Compute EI scores $\mathbf{s}^m$ and L2O baselines $\mathbf{b}^{m,\mathrm{L2O}}$.
            \STATE Set sequence-level advantages $\mathbf{A}^m=\mathbf{s}^m-\mathbf{b}^{m,\mathrm{L2O}}$.
        \ENDFOR
        \STATE Update $\theta$ using a group-based policy-gradient objective with advantages $\{\mathbf{A}^m\}_{m=1}^M$.
    \ENDWHILE
    \end{algorithmic}
\end{algorithm}

\subsection{Integration with RL Algorithms for LLM Training}
\label{sec:method:integration}

We integrate L2O-based advantage estimation into group-based RL for LLM post-training as \textbf{MaxPO} (Max@K Policy Optimization).
Algorithm~\ref{alg:maxpo} summarizes the procedure.
For each question $x_m$, we sample $G$ outputs $\{a_j^m\}_{j=1}^G$ from $\pi_\theta(\cdot \mid x_m)$ and evaluate their sequence-level rewards $\{r_j^m\}_{j=1}^G$.
Within each question-specific group, $G$ plays the role of the batch size $B$ used in the EI and L2O computations.
We compute group-wise EI scores $\mathbf{s}^m$ and L2O baselines $\mathbf{b}^{m,\mathrm{L2O}}$, and define the sequence-level advantage $A_j^m = s_j^m - b_j^{m,\mathrm{L2O}}$.
Before clipping and regularization, the corresponding group-wise estimator is $\frac{K}{G} \sum_{j=1}^{G} \nabla_\theta \log \pi_\theta(a_j^m \mid x_m) A_j^m$, which is an unbiased policy-gradient estimator for the max@\textit{K} objective within the group.
MaxPO is applicable to any policy-gradient algorithm that accepts sequence-level advantages.
In our LLM experiments, we instantiate it with group-based training \citep{shao2024deepseekmath, liu2025understanding}.
Since verifier rewards are assigned at the sequence level, we broadcast $A_j^m$ to every token in output $a_j^m$ when forming the token-level loss.
The resulting objective is identical to standard group-based clipped policy optimization except for the advantage construction.
EI-only methods use the raw EI scores $\mathbf{s}^m$, whereas MaxPO subtracts the L2O baseline and uses the centered advantages $\mathbf{s}^m-\mathbf{b}^{m,\mathrm{L2O}}$.
We give the exact token-level clipped objective and KL regularization term in Appendix~\ref{app:algo:maxpo}.

%% file: sec/unified.tex
\section{Canonical Finite-Batch Advantage for Max@K Policy Gradient}
\label{sec:unified_view}

Sec.~\ref{sec:method} derived MaxPO by adding an L2O baseline to the EI estimator of \citet{walder2025pass}, yielding an unbiased policy-gradient estimator with a centered advantage for max@\textit{K}.
However, EI is only one possible starting point for constructing pass@\textit{K}/max@\textit{K} advantages.
Recent methods use different primitives, including analytical pass@\textit{K} signals, all-subsets reward transformations, and standard-deviation-normalized advantages \citep{chen2025pass, bagirov2025best, tang2025optimizing}.
These differences make direct comparison difficult: the estimators may center different vectors, subtract different baselines, or apply different batch-dependent normalizations.

This section provides a unified baseline view of these estimators.
We first propose a canonical finite-batch advantage form based on the leave-one-out principle \citep{pmlr-v80-parmas18a,mnih2016asynchronous}.
We then show that existing estimators can be understood as uncentered signals, fixed-scale versions of this canonical direction, or normalized variants of it.

\subsection{Canonical Finite-Batch Marginal Advantage}
\label{subsec:canonical_marginal_contribution}

For a fixed prompt, let $\mathcal{B}=\{1,\ldots,B\}$ denote the sampled responses and define $M(S):=\max_{j\in S}r_j$ for any subset $S\subseteq\mathcal{B}$.
We define
\begin{equation}
    u_i
    :=
    \frac{1}{\binom{B-1}{K-1}}
    \sum_{\substack{S\subseteq\mathcal{B}\setminus\{i\}\\ |S|=K-1}}
    M(S\cup\{i\}),
    \qquad
    v_i
    :=
    \frac{1}{\binom{B-1}{K}}
    \sum_{\substack{T\subseteq\mathcal{B}\setminus\{i\}\\ |T|=K}}
    M(T).
\label{eq:unified_ui_vi}
\end{equation}
Here, $u_i$ is the conditional expected max@\textit{K} value of a group \textit{containing} response $i$.
It is the max@\textit{K} analogue of the return of action $i$.
The quantity $v_i$ is the leave-one-out expected max@\textit{K} value computed \textit{without} response $i$.
It is the corresponding leave-one-out baseline.
The signal $u_i-v_i$ is therefore the finite-batch max@\textit{K} analogue of a return minus a leave-one-out baseline.
Because $v_i$ excludes response $i$, subtracting it preserves policy-gradient unbiasedness.
Its centering is less immediate than in the ordinary return-minus-mean case, but follows by a finite-subset counting argument.
Let $V$ range over all size-$K$ subsets of $\mathcal{B}$.
When we sum $u_i$ over $i$, each subset $V$ contributes $M(V)$ exactly $K$ times, once for each $i\in V$:
\begin{equation}
    \sum_{i=1}^B u_i
    =
    \frac{K}{\binom{B-1}{K-1}}
    \sum_{\substack{V\subseteq\mathcal{B}\\ |V|=K}} M(V).
\end{equation}
When we sum $v_i$ over $i$, the same subset $V$ contributes exactly $B-K$ times, once for each $i\notin V$:
\begin{equation}
    \sum_{i=1}^B v_i
    =
    \frac{B-K}{\binom{B-1}{K}}
    \sum_{\substack{V\subseteq\mathcal{B}\\ |V|=K}} M(V).
\end{equation}
Since
$K/\binom{B-1}{K-1}=(B-K)/\binom{B-1}{K}$,
the two sums are equal, and hence $\sum_i (u_i-v_i)=0$ for every realized batch.
Thus, $u_i-v_i$ is the canonical centered finite-batch advantage; the full proof and its equivalence to EI+L2O are given in Appendix~\ref{app:unified}.

\subsection{Existing Estimators as Baseline Choices}
\label{subsec:existing_estimators_unified}

\begin{table}[t]
    \centering
    \small
    \caption{Baseline view of pass@\textit{K}/max@\textit{K} advantages.
        \checkmark\ indicates methods used in the original papers' experiments (for MaxPO, this refers to our experiments). Fixed constants are shown when they clarify equivalence up to learning-rate scaling for fixed $B$ and $K$.
    The methods used in the experiments fall into two categories: (1) \textcolor{softorange}{Uncentered}, which does not center the signal, and (2) \textcolor{softblue}{(Normalized) Canonical}, which is canonical up to a fixed scale.
    }
        \label{tab:unified_baseline_view}
    \makebox[\linewidth][c]{    \begin{tabular}{clll}
    \toprule
    Used & Method & Mathematical signal & Relation to canonical \\
    \midrule
    \checkmark & \citet{walder2025pass}, \texttt{EI-only}
    & $s_i = u_i - w_i$
    & \textcolor{softorange}{Uncentered} \\
    \checkmark & MaxPO
    & $s_i - b_i^{\mathrm{L2O}} = u_i - v_i$
    & \textcolor{softblue}{Canonical}. \\
    \checkmark & \citet{chen2025pass}, Eq. (14,15)
    & $\frac{B-K}{B}(u_i-v_i)/\sigma_{\mathrm{group}}$
    & \textcolor{softblue}{Normalized}. For pass@\textit{K} \\
    & \citet{chen2025pass}, w.o. std.
    & $\frac{B-K}{B}(u_i-v_i)$
    & Proportional. For pass@\textit{K} \\
    & \citet{bagirov2025best}, Eq. (9)
    & $\tilde r_i = \frac{K}{B}u_i$
    & Proportional to $u_i$ \\
    & \citet{bagirov2025best}, mean-centered
    & $\tilde r_i - \bar{\tilde r} = \frac{K(B-K)}{B^2}(u_i-v_i)$
    & Proportional \\
        \checkmark & \citet{bagirov2025best}, BoN mean, Appendix
    & $\frac{K(B-K)}{B^2}(u_i-v_i)/\mathrm{std}(\tilde r)$
    & \textcolor{softblue}{Normalized} \\
    \bottomrule
    \end{tabular}    }
\end{table}

Having identified the canonical coefficient, we can now relate the existing estimators to the canonical advantage signal.

\paragraph{EI-based advantage estimators.}
To connect EI with the marginal view, define $w_i := \frac{1}{\binom{B-1}{K-1}} \sum_{\substack{S\subseteq\mathcal{B}\setminus\{i\}\\ |S|=K-1}} M(S)$
the expected maximum of the $K-1$ comparator responses excluding $i$.
The EI-only estimator can then be written as $s_i=u_i-w_i$, because
$M(S\cup\{i\})=M(S)+(r_i-M(S))_+$.
Thus, EI subtracts the comparator-max baseline $w_i$ rather than the leave-one-out max@\textit{K} baseline $v_i$.
This is sufficient for policy-gradient unbiasedness, because $w_i$ excludes response $i$, but it is not the canonical centered advantage.
In general, $u_i-w_i$ is nonnegative and non-centered.

MaxPO corrects this difference via the L2O baseline.
Indeed, we can show
\begin{proposition}\label{prop:unified:ei_l2o_equivalence}
        Given a batch $\mathcal{B}$ and a subset size $K$, the L2O baseline $b_i^{\mathrm{L2O}}$ is equal to $v_i - w_i$ and therefore satisfies
    \begin{equation}
        s_i - b_i^{\mathrm{L2O}} = u_i - v_i,
    \end{equation}
    for any $i \in \mathcal{B}$.
\end{proposition}
The proof is provided in Appendix~\ref{app:unified}.
This equality shows that L2O rectifies the EI-based advantage estimator to the canonical centered advantage.

\paragraph{Generality of L2O.}
This cancellation is not specific to choosing EI as the starting signal.
Suppose the starting signal for response $i$ is an exact average over groups containing $i$ of $M(S\cup\{i\})-C(S)$, where $C(S)$ is any comparator-only baseline depending on the other $K-1$ responses $S$ but not on response $i$.
Then applying the same L2O construction cancels the averaged $C(S)$ term and again yields $u_i-v_i$.
Thus, L2O maps any exact comparator-only baseline of this form to the canonical centered finite-batch advantage; Appendix~\ref{app:unified:arbitrary_comparator_baselines} gives the proof.

\paragraph{Non-EI estimators.}
The same view clarifies how non-EI estimators \citep{chen2025pass, bagirov2025best} relate to MaxPO.
Table~\ref{tab:unified_baseline_view} summarizes how existing estimators defined with $K\leq B$ relate to the canonical direction.
The table suggests that the estimators used in their experiments \citep{walder2025pass, chen2025pass, bagirov2025best} fall into two classes: \emph{uncentered} and \emph{(normalized) canonical}.
The first class is \emph{uncentered}: PKPO/EI-only is policy-gradient unbiased but uses the non-centered signal $u_i-w_i$.
The second class is \emph{(normalized) canonical}, or canonical up to a fixed scale: MaxPO exactly recovers $u_i-v_i$, while raw analytical pass@\textit{K} and mean-centered all-subsets estimators recover the same direction up to fixed constants.
Thus, for the analytical signal of Chen et al.\ and the mean-centered all-subsets signal of Bagirov et al., what is biased relative to MaxPO is the magnitude, not the direction; in the standard-deviation-normalized versions used in their experiments, this fixed positive scale is normalized away or absorbed into the effective step size, so it does not change the normalized update direction.
This view reveals that although the derivations are different, some estimators are equivalent to MaxPO up to a fixed scale.

%% file: sec/unified_experiment.tex
\section{Experiments}\label{sec:unified_experiment}\label{sec:toy_experiment}

In Sec.~\ref{sec:toy_experiment:bandits}, we validate the theoretical properties of the L2O baseline in bandits.
In Sec.~\ref{sec:reasoning_experiments}, we validate the efficacy of MaxPO on reasoning tasks.

\subsection{Bandits}\label{sec:toy_experiment:bandits}

\begin{figure}[t]
    \vspace{-1.0em}
    \centering
    \includegraphics[width=0.9\textwidth]{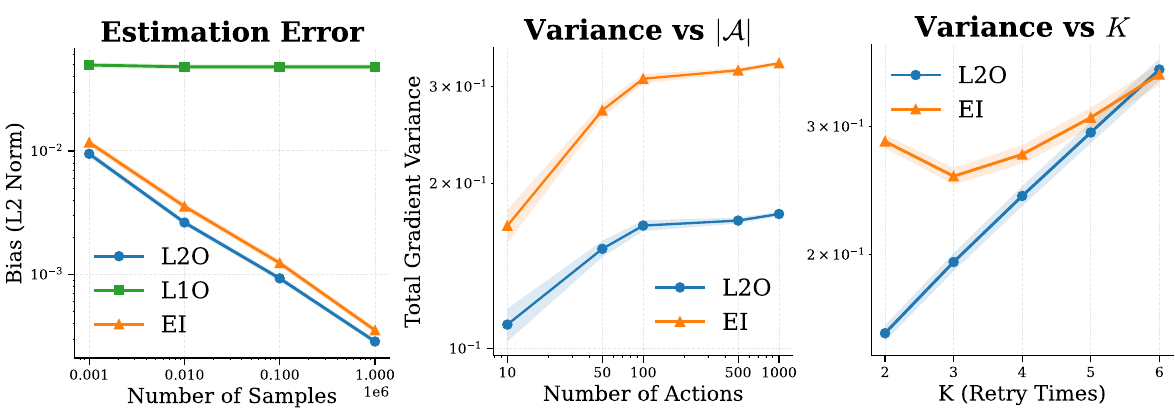}
    \caption{
        Estimation error (left), variance vs.\ action space size (center), and variance vs.\ $K$ (right).
        Mean and standard error over 100 random seeds.
    }
    \label{fig:bandit_bias_variance}
\end{figure}

Here, we validate the two theoretical properties of the L2O baseline: (1) unbiasedness as a PG estimator and (2) variance reduction over the raw EI estimator due to centering the advantage.

\paragraph{Setting.}
To emulate a discrete-action setting such as LLM training, we consider a multi-armed bandit with rewards $\bfr \in \R^{|\cA|}$ and logits $\bfl \in \R^{|\cA|}$ each sampled from $\cN(0, 1)$.
We sample a batch of $B$ actions $a_i \sim \mathrm{softmax}(\bfl)$, observe rewards $r_i$, and construct three PG estimators: (1) EI: $s_i$, (2) EI+L2O: $s_i - b^{\mathrm{L2O}}_{-i}$, and (3) EI+L1O: $s_i - b^{\mathrm{L1O}}_{-i}$, as defined in Sec.~\ref{sec:method}.
The ground-truth gradient $g_{\text{true}}$ is computed analytically via the closed-form derivative of the expected improvement (Proposition~1 of \citet{nishimori2026emergence}).
For the bias plot, we vary the number of batches $N \in \{10^3, 10^4, 10^5, 10^6\}$ and measure the estimation error $\|\frac{1}{N}\sum_{j=1}^N \hat{g}_j - g_{\text{true}}\|$ (fixing $B=8$, $K=2$).
For variance, we report the empirical total variance $\frac{1}{N}\sum_{j=1}^N \|\hat{g}_j - \bar g\|^2$ with $\bar g = \frac{1}{N}\sum_{j=1}^N \hat{g}_j$ and $N=10^5$, sweeping (i) the action space size $|\cA|\in\{10, 50, 100, 1000\}$ at $K=2, B=8$, and (ii) the comparator size $K \in \{2,\dots,6\}$ at $|\cA|=100, B=8$.
Detailed protocols and results for additional batch sizes are reported in Appendix~\ref{app:toy_experiments:bandits}.
\paragraph{Results.}
The results in Fig.~\ref{fig:bandit_bias_variance} validate our theoretical claims.
First, the estimation errors of EI+L2O and EI both decrease at the theoretical $\cO(1/\sqrt{N})$ rate (left), confirming that they are unbiased PG estimators, whereas EI+L1O exhibits an approximately constant error, confirming that L1O yields a biased PG.
Regarding variance, the variance-reduction effect of L2O over raw EI grows with the action space size (center), a regime particularly relevant for LLMs.
Furthermore, L2O is most effective for moderate $K$ relative to $B=8$ (right): the L2O baseline must form $K-1$ comparators from $B-2$ samples to preserve unbiasedness, so a too-small comparator budget yields a noisy baseline and diminishes the variance-reduction effect.

\subsection{LLM Reasoning Experiments}
\label{sec:reasoning_experiments}

We evaluate the efficacy of our variance-reduced objective on challenging math reasoning tasks.
Our experiments address two questions:
(i) whether our estimator improves pass@\textit{k} on standard reasoning benchmarks, and
(ii) whether it stabilizes RL optimization by reducing the variance of the gradient estimator during training.

\begin{figure}[t]
  \centering
  \begin{subfigure}[t]{0.45\linewidth}
    \centering
    \includegraphics[width=\linewidth]{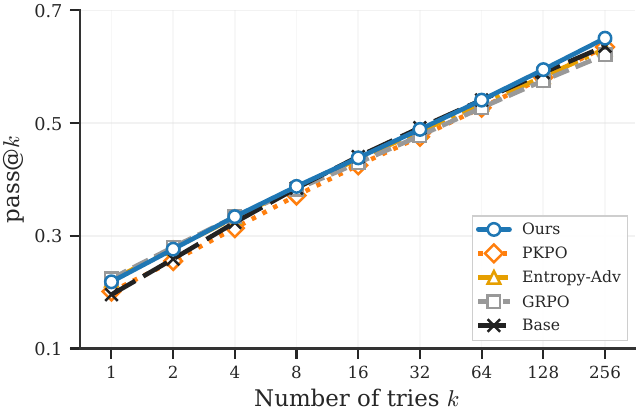}
                                    \caption{Llama-3.2-3B-Instruct}
    \label{fig:passk_avg_256_llama}
  \end{subfigure}\hfill
  \begin{subfigure}[t]{0.45\linewidth}
    \centering
    \includegraphics[width=\linewidth]{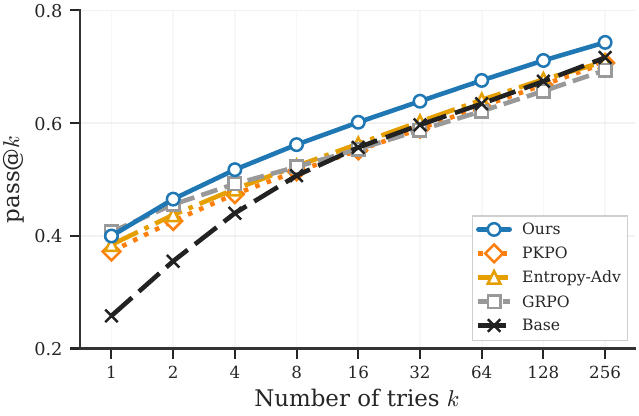}
                                    \caption{Qwen2.5-Math-7B}
    \label{fig:passk_avg_256_qwen}
  \end{subfigure}

  \caption{Task-average pass@k ($k \le 256$).
  Unweighted average over AIME24, AIME25, AMC23, MATH500, and Minerva (temperature 0.6, top-p 0.95).
  Our method demonstrates consistent improvement over strong baselines. }
  \label{fig:passk_avg_256}
\end{figure}

\subsubsection{Experimental Setting}
\label{sec:reasoning_setup}

\paragraph{Training.}
We perform RL fine-tuning on Llama-3.2-3B-Instruct~\citep{grattafiori2024llama3} and Qwen2.5-Math-7B~\citep{yang2025qwen25math}.
We compare our method (EI+L2O) against three baselines:
GRPO~\citep{shao2024deepseekmath},
Entropy-Adv~\citep{cheng2025reasoning},
and PKPO~\citep{walder2025pass}.
Our method follows the algorithm described in Sec.~\ref{sec:method:integration}, and we implement the LLM RL training pipeline using the \texttt{verl} framework~\citep{sheng2024hybridflow}.
We include PKPO as the closest prior method that directly optimizes pass@\textit{K}/max@\textit{K} using an unbiased expected-improvement (EI) policy-gradient estimator; in our notation, PKPO can be viewed as the EI-only variant, so this comparison isolates the benefit of our centered advantage estimation.
Note that the concurrent work of \citet{bagirov2025best} and prior work of \citet{chen2025pass} are equivalent to MaxPO up to a
constant rescaling factor and standard deviation normalization (which removes any effect from rescaling). Therefore, these works
also provide experimental evidence for the overall efficacy of the approach. Our contribution relative to the literature is
the canonical $u_i-v_i$ unbiased form of the gradient estimator, thus our experiments focus on how this form improves over the non-centered PKPO variant.
Unless otherwise stated, we use a fixed training objective size $K=2$.
Training data and hyperparameters are provided in App.~\ref{app:llm_experiments:setting}.
Throughout this section, $K$ refers to the training objective size (max@\textit{K}/pass@\textit{K} objective), while $k$ denotes the evaluation compute in pass@\textit{k}.

\paragraph{Evaluation.}
We evaluate on five math reasoning benchmarks: AIME24, AIME25, AMC23, MATH500~\citep{hendrycks2021math}, and Minerva~\citep{lewkowycz2022solving}.
All evaluations use nucleus sampling with temperature 0.6 and top-p 0.95.
To reduce evaluation variance, we generate $n=1024$ samples for every benchmark.
We report the unbiased pass@\textit{k}~\citep{chen2021evaluating} metric for $k \in \{1,2,4,8,\ldots\}$, computed as
\begin{equation}
\text{pass@}k \;:=\; \mathbb{E}_{x\sim \mathcal{D}}\!\left[\,1 - \frac{\binom{n-c}{k}}{\binom{n}{k}}\,\right],
\label{eq:passk}
\end{equation}
where $n$ is the number of sampled completions and $c$ is the number of correct completions among them.
For consistent task-averaging in the main text, we report results up to $k\le256$ for all benchmarks,
and additionally report $k\le1024$ for AIME24, AIME25, and AMC23 in Appendix~\ref{app:llm_experiments:additional_results}.

\subsubsection{Main Results}
\label{sec:reasoning_main_results}

Figure~\ref{fig:passk_avg_256} shows the task-average pass@\textit{k} curves up to $k=256$ for both Llama and Qwen.
The main pattern is that EI+L2O becomes strongest as inference compute increases: although GRPO is competitive at very small $k$, our method overtakes the baselines at moderate-to-large $k$, which is the regime most relevant to pass@K optimization.
Compared to PKPO, our method improves task-average pass@256 by 5.2\% on Qwen2.5-Math-7B and by 2.4\% on Llama-3.2-3B-Instruct relative to PKPO across five math reasoning benchmarks.
These gains persist over a wide range of $k$, supporting the claim that centering the advantage helps the policy make better use of additional test-time samples.
We provide extended results up to $k=1024$ for AIME24, AIME25, and AMC23 in Appendix~\ref{app:llm_experiments:additional_results}.
For an ablation over different $K$ values, refer to Appendix~\ref{app:ablation_k}.

\subsubsection{Variance Reduction During Training}
\label{sec:variance_reduction_training}

\begin{wrapfigure}{r}{0.4\textwidth}
  \vspace{-1.0em}
  \centering
  \includegraphics[width=0.4\textwidth]{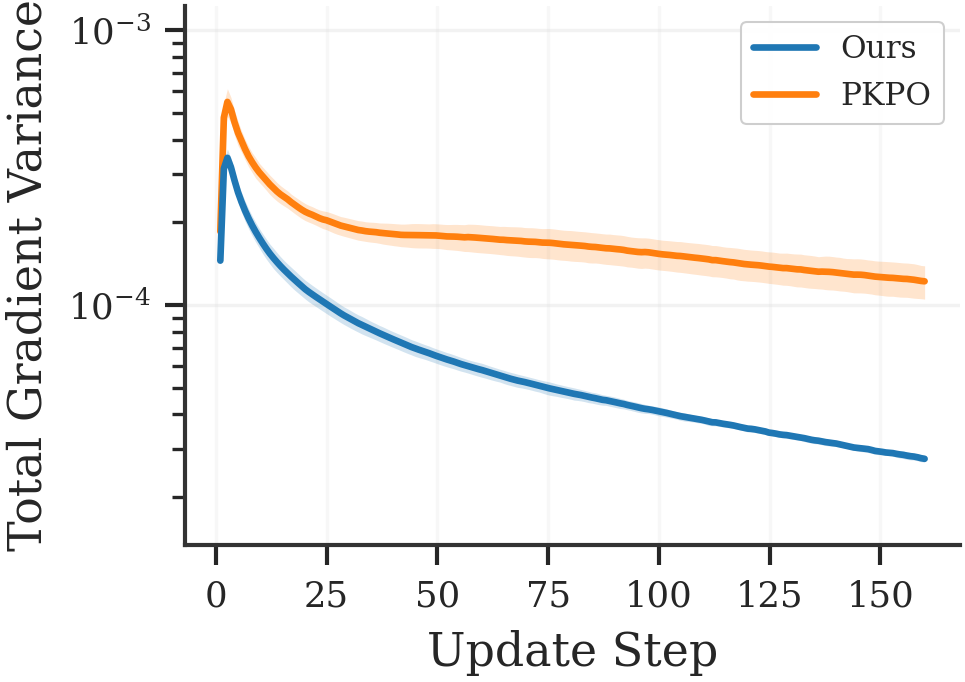}
    \caption{Adam-moment variance proxy during training (3 seeds).
  The proxy is estimated from Adam states via $\mathrm{Var}(g)\approx \hat v_t-\hat m_t^{\,2}$ and aggregated across parameters.
  Our method (EI+L2O) reduces this variance proxy compared to PKPO across training, supporting the proposed variance-reduction mechanism.}
  \label{fig:grad_var_3seeds}
  \vspace{-1.0em}
\end{wrapfigure}

To test our theoretical motivation, we measure a smoothed proxy for the variance of model gradients during RL training on Llama-3.2-3B-Instruct.
Directly computing gradient variance is expensive at LLM scale, so we estimate this proxy from the states of the Adam optimizer~\citep{kingma2017adam}.
Specifically, Adam maintains exponential moving averages of the first and second moments of gradients; after bias correction, we estimate a smoothed proxy for the element-wise variance as
$\mathrm{Var}(g)\approx \hat v_t - \hat m_t^{\,2}$,
and aggregate it across all parameters by summing.
This is not the exact instantaneous gradient variance, but it is practical at LLM scale and computed identically for all methods, making it suitable for comparison.

Figure~\ref{fig:grad_var_3seeds} reports mean $\pm$ standard deviation over three random seeds.
Our method yields consistently lower values of this Adam-moment variance proxy than PKPO throughout training, indicating more stable policy gradient updates.
Notably, at the end of training, our method achieved a 77.4\% relative reduction in this proxy compared to PKPO.
This provides empirical evidence that centering the EI signal with L2O reduces estimator variance in practice, matching the paper's main theoretical claim.

%% file: sec/conclusion.tex
\section{Conclusion}\label{sec:conclusion}
In this paper, we revisit pass@\textit{K} and max@\textit{K} policy optimization through the lens of advantage estimation, with a focus on unbiased advantage estimation.
We first identified that the advantage estimator of \citet{walder2025pass}, a leading method in the field, is unbiased as a PG estimator but not centered as an advantage estimator.
We then introduced a Leave-Two-Out (L2O) baseline that preserves policy-gradient unbiasedness while making the realized batch advantages exactly centered.
We also derived the canonical finite-batch advantage form that enables us to organize and relate existing estimators in a unified manner, classifying them into uncentered, canonical, and normalized canonical classes.
Finally, we empirically verified that the L2O baseline reduces gradient variance and outperforms non-centered alternatives.

\paragraph{Limitations and Future Work.}
Our analysis focuses on finite-batch, on-policy estimation under i.i.d. sampling within each group.
Extending the same canonical view to off-policy training, adaptive sampling, and settings with correlated generations is an important direction for future work.
Our L2O construction also requires $K \le B - 1$, which may limit its direct use when the optimization objective size approaches the group size.
In addition, our LLM experiments focus on mathematical reasoning tasks with verifiable sequence-level rewards.
Further experiments on broader domains, noisy or partial rewards, larger models, and different verifier designs are needed to understand the generality of the method.

\paragraph{Broader Impacts.}
Max@\textit{K} and pass@\textit{K} are increasingly important metrics for reasoning models because they capture the quality of a model under repeated sampling, which is a common inference-time protocol.
By identifying a canonical finite-batch advantage form for max@\textit{K}, this work provides a principled foundation for designing, comparing, and diagnosing policy-gradient estimators for such inference-time objectives.
This may help improve the efficiency and stability of post-training methods and reduce unnecessary variance in optimization.
At the same time, better optimization of repeated-sampling reasoning objectives may strengthen models that can be misused for large-scale cheating, deceptive assistance, or other forms of harmful automated problem solving.

\section*{Author Contributions}

{\bf Shota Takashiro}\textsuperscript{*}: Equal first author. LLM experiment lead, paper writing (particularly LLM sections) \\[0.25em]
{\bf Soichiro Nishimori}\textsuperscript{*}: Equal first author. Lead writer, toy experiments (bandit, maze) \\[0.25em]
{\bf Paavo Parmas}\textsuperscript{*}: Equal first author. Conceptualization, all theoretical derivations, proposed and oversaw project, significant comments and editing on the paper, example code. \\[0.25em]
Yongmin Kim: Helped on the LLM engineering side, L1O implementation, tests. \\[0.25em]
Kohsei Matsutani: Helped with the LLM experiments, dataset survey, evaluation results, contributed to LLM related work sections.\\[0.25em]
Gouki Minegishi: Entropy bonus baseline implementation, gradient variance estimation from Adam implementation. \\[0.25em]
Yusuke Iwasawa: Funding acquisition, overall management and student supervision in the lab. \\[0.25em]
Takeshi Kojima: Comments and writing, particularly from the LLM perspective. \\[0.25em]
Yutaka Matsuo: Funding acquisition, overall management and student supervision in the lab.

\section*{Acknowledgements}
Paavo Parmas was supported by JST ACT-X, Japan, Grant Number JPMJAX23CO. This work was supported by the UTokyo-Google AI Symbiotic Future Society Program.

%% file: sec/app/additional_related_work.tex
\section{Additional Related Work}\label{app:additional_related_works}
\subsection{RLVR in LLMs}\label{sec:related_work:rlvr}
Reinforcement learning with verifiable rewards (RLVR) \citep{lambert2025tulu,guo2025deepseek} is a training paradigm that optimizes LLM policies in domains such as math and code using deterministic feedback from objective verifiers to ground model reasoning in provable correctness rather than subjective human preference \citep{paul2017deep}. Despite this promise \citep{guo2025deepseek,jaech2024openai}, recent studies suggest that RLVR in LLMs primarily serves to amplify behaviors that are already present in the base model \citep{liu2025understanding,zhao2025echochamber,ai2025rethinking}. \citet{yue2025does} investigated the pass@K metric \citep{chen2021evaluating,song2025mind,dang2025weight,wen2025reinforcement,wu2025invisible}, which measures the probability that at least one correct solution is obtained when drawing $K$ independent samples (i.e., best-of-$K$), and found experimentally that as $K$ increases, the base model's pass@K eventually exceeds that of the RLVR-trained model. This phenomenon has been linked to diversity collapse \citep{dang2025weight,cui2025entropy} and squeezing reasoning paths \citep{matsutani2025rl,hu2025how,bu2025consistency}. Several studies \citep{wang2025octothinker,zhang2025learning,chen2025the,zhang2025interplay} argue that mid-training is important. Consistent with this view, empirical results suggest that RLVR underperforms on Llama models \citep{grattafiori2024llama3} relative to Qwen models \citep{yang2025qwen25,yang2025qwen3}.

\subsection{Exploration in RL for LLMs}\label{sec:related_work:exploration}
In light of these limitations, prior studies have incorporated exploration into RLVR. \citet{cui2025entropy,cheng2025reasoning,zheng2025first,shen2025entropy,jiang2025rethinking} leveraged entropy bonuses to encourage exploration via policy uncertainty. \citet{song2025outcome} proposed outcome-based exploration using UCB-style bonuses \citep{auer2002finite}. \citet{yu2025restrain} applied self-penalization by assigning negative rewards to high-confidence answers that deviate from the majority consensus. \citet{he2025rewarding} improved exploration by up-weighting low-probability but correct trajectories, and \citet{gao2025navigate} adopted Random Network Distillation (RND) \citep{burda2018exploration} to provide bonuses for unknown trajectories. \citet{zhou2025evolving} augmented training with a semantic novelty score computed from embeddings, \citet{li2025jointly} employed a semantic diversity score with an external semantic comparator, and \citet{tuyls2025representation} computed a representation-based novelty score from hidden states to boost exploration. \citet{liang2025can} leveraged reward-model gradients to improve temperature sampling. \citet{setlur2025e3} promoted in-context exploration via skill asymmetries and negative gradients, enabling reliable extrapolation with increased test-time compute.

Other studies directly optimize the pass@K metric; these are discussed in App.~\ref{sec:related_work:pass@K}.

\subsection{Policy Gradient Estimator and Baseline}
Policy gradients are a central approach for optimizing policies in RL \citep{Williams1992-rp,greensmith2004variance,parmas2018total,parmas2021unified}.
The high variance of gradient estimates is a fundamental challenge in policy gradient methods, often hindering stable convergence \citep{Williams1992-rp}.
To mitigate this, the method of control variates—subtracting a baseline $b(x)$ from the return—is the standard variance reduction technique.
Theoretically, the optimal baseline depends on the norm of the score function and the specific reward structure \citep{peters2008reinforcement, weaver2013optimal}.
\citet{greensmith2004variance} provided a comprehensive analysis, demonstrating that while the exact optimal baseline is computationally expensive, a baseline that approximates the expected return—thereby making the advantage approximately zero-mean—captures the majority of the variance reduction benefits.
This insight justifies the widespread adoption of value function baselines in modern algorithms like PPO \citep{schulman2017proximal} and A3C \citep{mnih2016asynchronous}.

In domains involving discrete latent variables or sequence generation, where learning a separate value function is often unstable or costly, \textit{sample-based baselines} have become the dominant approach.
This concept was further refined in the context of variational inference by \citet{mnih2016variational} (VIMCO) and \citet{gu2015muprop}, which utilize the average reward of other samples in the batch (Leave-One-Out; L1O) to construct a low-variance gradient estimator for discrete variables.
Specifically, \citet{tucker2017rebar} highlighted that such control variates are essential for training effective estimators in high-dimensional discrete spaces.

Recently, these multi-sample baseline techniques have been adapted for reasoning tasks in LLMs.
Group Relative Policy Optimization (GRPO) \citep{shao2024deepseekmath} applies group-based normalization, effectively an L1O baseline, to stabilize training without a critic network.
However, these standard L1O techniques rely on the linearity of the expectation operator.
\citet{wu2018variance} discussed the complexities of action-dependent baselines, but the specific challenges of the \textit{max@K} objective remain underexplored.
In the max@K setting, the non-linear dependency between samples introduced by the max operator renders standard L1O baselines biased, necessitating the development of our strictly unbiased Leave-Two-Out (L2O) approach.

\subsection{Pass@K Policy Optimization}\label{sec:related_work:pass@K}
Recent advancements in reasoning tasks have increasingly focused on directly optimizing the \textit{pass@K} metric, an evaluation criterion where $K$ independent samples are drawn from the model, and success is defined by at least one sample being correct \citep{tang2025optimizing, walder2025pass, peng2025simko, bagirov2025best, chen2025pass}. The primary objective of this approach is to maintain model diversity, thereby encouraging exploration to discover correct solutions.
\citet{tang2025optimizing} first proposed optimizing the pass@K objective using RL, while \citet{chen2025pass} provided an empirical analysis of its effects on model behavior, particularly regarding entropy and generation diversity.
Most relevant to our work is \citet{walder2025pass}, who generalize the pass@K objective to continuous rewards (denoted as \textit{max@K}) and propose multiple estimators for this objective.
However, we observe that the estimators proposed by \citet{walder2025pass} can lead to overestimated advantage terms.
To address this, we re-examine the max@K objective and propose a novel estimator designed to ensure the advantage function has an expected value of zero, thereby stabilizing optimization.

Prior to the development of pass@K policy optimization, the RL community has explored optimizing expected rewards across multiple trials \citep{koyamada2023emergence}.
\citet{koyamada2023emergence} formally introduced the \textit{ReMax} objective to maximize the max@K outcome, originally optimizing it within a resettable simulator.
Subsequently, \citet{nishimori2026emergence} extended this framework by deriving a policy gradient estimator.
We are inspired by their expected-improvement form of the policy-gradient estimator, which enables us to analyze the advantage term of the max@K policy gradient.
However, a direct application is infeasible in the language model setting, where computing rewards for the entire action space is intractable.
Therefore, we propose a practical estimator that approximates the EI using only a limited batch of samples.

%% file: sec/app/code.tex
\newpage
\section{Code to Compute the Estimators}\label{app:code}

Below, we provide code to compute the statistics proposed in Sec.~\ref{sec:method}, namely the EI estimator $s_i$ and the L2O baseline $b_{-i}^{\mathrm{L2O}}$.

\begin{lstlisting}[caption=Code to compute the statistics proposed in Sec. \ref{sec:method}, label={lst:code:statistics}]
import jax.numpy as jnp
from scipy.special import gammaln

def comb(n, k):
    """Computes binomial coefficient C(n, k) in log-space."""
    return jnp.exp(gammaln(n + 1) - gammaln(k + 1) - gammaln(n - k + 1))

def compute_batch_ei(returns: jnp.ndarray, K: int) -> jnp.ndarray:
    """
    Computes the unbiased EI estimator s_i using matrix operations.
        Corresponds to Theorem E.1.
    """
    B = returns.shape[0]
    # Sort rewards: r_{(1)} <= ... <= r_{(B)}
    order = jnp.argsort(returns)
    r_sorted = returns[order]

    # Compute Difference Matrix D_{i,j} = (r_{(i)} - r_{(j)})_+
    D = jnp.maximum(r_sorted[:, None] - r_sorted[None, :], 0.0)

        # Compute Weight Vector w_j (Theorem E.1)
    # w_j = C(j-1, K-2) / C(B-1, K-1)
    ranks = jnp.arange(B)
    w = comb(ranks, K - 2) / comb(B - 1, K - 1)
    w = jnp.nan_to_num(w) # Handle cases where rank < K-2
    # s = D @ w
    s_sorted = D @ w
    # Restore original order
    return s_sorted[jnp.argsort(order)]

def compute_l2o_baseline(returns: jnp.ndarray, K: int) -> jnp.ndarray:
    """
        Computes the L2O baseline b_{-i}^{L2O} (Theorem E.2) in O(B^2).
    """
    B = returns.shape[0]
    order = jnp.argsort(returns)
    r_sorted = returns[order]

    # Difference Matrix D and Column Sums S_l
    D = jnp.maximum(r_sorted[:, None] - r_sorted[None, :], 0.0)
    S = jnp.sum(D, axis=0) # S_l = sum_k D_{k,l}

        # LOO Column Mean Matrix M_{i,l} (Theorem E.2)
    # M_{i,l} = (S_l - D_{i,l}) / (B - 1)
    M = (S[None, :] - D) / (B - 1)

    # L2O Weight Matrix W^{L2O}_{i,l}
    # Rank adjustment: (l-1) - 1 if l > i else (l-1)
    i_idx = jnp.arange(B)[:, None]
    l_idx = jnp.arange(B)[None, :]
    adjusted_rank = l_idx - (l_idx > i_idx).astype(jnp.float32)

    W = comb(adjusted_rank, K - 2) / comb(B - 2, K - 1)
    W = jnp.nan_to_num(W)
    W = W * (1.0 - jnp.eye(B)) # Set diagonal (l=i) to 0

    # b_{-i}^{L2O} = sum_l M_{i,l} * W_{i,l}
    b_sorted = jnp.sum(M * W, axis=1)
    return b_sorted[jnp.argsort(order)]
  \end{lstlisting}

%% file: sec/app/algo.tex
\begin{algorithm}[t]
    \caption{Max@K Policy Optimization}
    \label{alg:grpo_ei_l2o}
    \begin{algorithmic}[1]
    \STATE \textbf{Input:} Policy $\pi_\theta$, objective size $K$, the number of questions per batch $M$, group size $G$ with $2 \le K \le G-1$.
    \WHILE{not converged}
    \STATE Sample $M$ questions $\{x_m\}_{m=1}^M$.
    \FOR{each question $x_m$}
    \STATE Sample $G$ outputs $\{a_j\}_{j=1}^G \sim \pi_\theta(\cdot \mid x_m)$ and evaluate rewards $\{r_j\}_{j=1}^G$.
    \STATE Compute EI scores $\mathbf{s}$ (Theorem~\ref{thm:batch_ei}) and the L2O baseline $\mathbf{b}^{\mathrm{L2O}}$ (Theorem~\ref{thm:l2o_baseline}).
    \STATE Compute advantages: $\mathbf{Adv} = \mathbf{s} - \mathbf{b}^{\mathrm{L2O}}$.
    \ENDFOR
    \STATE Update $\theta$ by maximizing Eq.~\eqref{eq:grpo_ei_l2o_loss} with $\mathbf{Adv}$.
    \ENDWHILE
    \end{algorithmic}
\end{algorithm}

\section{Details of Max@K Policy Optimization}\label{app:algo:maxpo}
We integrate our method into modern RL algorithms for LLM post-training.
Algorithm~\ref{alg:grpo_ei_l2o} summarizes the overall procedure when combining EI + L2O with group-based RL, which we call \textbf{Max}@K \textbf{P}olicy \textbf{O}ptimization (\textbf{MaxPO}).
In Sec.~\ref{sec:method:ei} and Appendix~\ref{app:computation}, $B$ denotes the number of samples used to compute EI and L2O for a single input.
In LLM training, this computation is applied independently within each question-specific group.
Concretely, for each question $x_m$, we sample $G$ outputs $\{a_j^m\}_{j=1}^G$ from $\pi_\theta(\cdot \mid x_m)$ and evaluate their sequence-level rewards $\{r_j^m\}_{j=1}^G$.
Within such a group, $G$ plays the role of $B$ from Sec.~\ref{sec:method:ei} and Appendix~\ref{app:computation}.
From these rewards, we compute group-wise EI scores $\mathbf{s}^m$ and L2O baselines $\mathbf{b}^{m,\mathrm{L2O}}$.
Accordingly, for each question-specific group, the estimator $\frac{K}{G} \sum_{j=1}^G \nabla_{\theta} \log \pi_\theta(a_j^m \mid x_m)\bigl(s_j^m - b_{-j}^{m,\mathrm{L2O}}\bigr)$ is unbiased as a PG estimator.
Thus, our method acts as a reward transformation from raw sequence rewards to centered sequence-level advantages $(s_j^m - b_{-j}^{m,\mathrm{L2O}})_{j=1}^G$, which optimize the max@K objective (Eq.~\eqref{eq:max@K}).

While applicable to any PG-based algorithm, we instantiate our method using group-based RL algorithms such as Group Relative Policy Optimization (GRPO) \citep{shao2024deepseekmath} and Dr.GRPO \citep{liu2025understanding} for their success in reasoning tasks.
An optimization batch contains $M$ questions, yielding question-specific groups $(x_m, a^m_{1:G}, r^m_{1:G})_{m=1}^M$.
Each output text consists of a sequence of tokens.
Because rewards in reasoning tasks are typically assigned at the sequence level, we first compute a sequence-level advantage $A^m_j = s^m_j - b_{-j}^{m,\mathrm{L2O}}$ for each output text.
We then broadcast this same scalar to all token positions when forming the token-level loss, i.e., $A^m_{j,i} = A^m_j$ for every token position $i$ in output $a^m_j$.
In practice, this corresponds to assigning the same reward to all tokens in an output text, $r^m_{j,i} = r^m_j$.
Given a reference policy $\pi_{\text{ref}}$, we optimize the following loss:
\begin{align}\label{eq:grpo_ei_l2o_loss}
    &L(\theta) = \frac{1}{M} \sum_{m=1}^M \frac{K}{G} \sum_{j=1}^G \frac{1}{|a^m_{j}|} \sum_{i=1}^{|a^m_{j}|}
    &\min \left[
    \phi_\theta(x_m, a^m_{j,i}) A^m_{j,i}, \phi^\epsilon_\theta(x_m, a^m_{j,i}) A^m_{j,i} \right] - \beta \mathrm{KL}(\pi_\theta | \pi_{\text{ref}}),
\end{align}
where $\beta \geq 0$ and $\epsilon \geq 0$ are hyperparameters, and $\mathrm{KL}$ denotes the Kullback--Leibler (KL) divergence.
We define \mbox{$\phi_\theta(x, a_i)=\pi_\theta(a_i\mid x, a_{<i})/\pi_{\text{ref}}(a_i\mid x, a_{<i})$} as the ratio between our policy and the reference policy, and $\phi^\epsilon_\theta(x, a_i)$ as its clipped version in $[1 - \epsilon, 1 + \epsilon]$.
$|a^m_{j}|$ is the sequence length of the $j$-th text in the $m$-th group.
Eq.~\eqref{eq:grpo_ei_l2o_loss} is identical to GRPO \citep{shao2024deepseekmath} except for how the advantage is computed.
While standard GRPO normalizes rewards by dividing by the inner-group standard deviation, a practice that introduces bias into the PG estimator, both our method and Pass@K Policy Optimization (PKPO) \citep{walder2025pass} avoid this division to preserve unbiasedness.
The key distinction from \citet{walder2025pass} lies in the advantage construction: while their approach is equivalent to using our raw EI scores, we further subtract the L2O baseline to achieve lower variance.

%% file: sec/app/proof.tex
\section{Proofs}\label{app:proof}

\subsection{Proof of Proposition \ref{prop:l2o_baseline_unbiasedness}}\label{app:proof:l2o_baseline_unbiasedness}
In this appendix, we prove both claims in Proposition \ref{prop:l2o_baseline_unbiasedness}: unbiasedness of the L2O advantage in expectation and exact centering of its realized batch mean.

\begin{proof}
We assume $2 \le K \le B-1$ so that all subsets used below are well-defined.
Let $\cD = (a_{1:B}, r_{1:B})$ be a set of samples where each action is drawn i.i.d. from $\pi_\theta$ and $r_i := r(a_i)$.
We define the \textit{population Expected Improvement} $\mu_{EI}$ as the expected gain of a single action against the maximum of $K-1$ other independent samples.
Formally, let $a_1, \dots, a_K \overset{\iid}{\sim} \pi_\theta$ be $K$ independent random variables.
Due to the i.i.d. assumption on $a_1, \dots, a_K$, the population EI is defined as:
\begin{equation}\label{eq:app:population_ei}
\mu_{EI} := \E[a]{s(a)} = \E[a_{1:K}]{ \left(r(a_1) - \max_{k=2,\dots,K} r(a_k)\right)_+ }.
\end{equation}

\paragraph{Unbiasedness of the EI Estimator $s_i$.}
First, we verify that $s_i$ targets $\mu_{EI}$ following the U-statistics theory \citep{hoeffding1992class}.
Given a batch $\cD = (a_{1:B}, r_{1:B})$, $s_i$ is constructed by averaging over all possible subsets of size $K-1$ from the batch excluding $i$ (denoted as $\cU_{-i} = \{1, \dots, B\} \setminus \{i\}$).
\begin{equation}
s_i := \E[\cI]{ \left(r_i - \max_{k \in \cI} r_k\right)_+ } = \frac{1}{\binom{B-1}{K-1}} \sum_{\cI \subseteq \cU_{-i}, |\cI|=K-1} \left(r_i - \max_{k \in \cI} r_k\right)_+.
\end{equation}
Taking the expectation over random batches $\cD \sim \pi_\theta$:
\begin{equation}
\E[\cD \sim \pi_\theta]{s_i} = \frac{1}{\binom{B-1}{K-1}} \sum_{\cI \subseteq \cU_{-i}, |\cI|=K-1} \E[\cD \sim \pi_\theta]{ \left(r_i - \max_{k \in \cI} r_k\right)_+ }.
\end{equation}
For any fixed subset $\cI$ of size $K-1$, the set of indices $\{i\} \cup \cI$ constitutes $K$ distinct samples.
Again, by the i.i.d. assumption, the joint distribution of $\{a_i\} \cup \{a_k\}_{k \in \cI}$ depends only on the number of samples, not their indices.
Thus:
\begin{equation}
\E[\cD \sim \pi_\theta]{ \left(r_i - \max_{k \in \cI} r_k\right)_+ } = \E[a_{1:K}]{ \left(r(a_1) - \max_{k=2,\dots,K} r(a_k)\right)_+ } = \mu_{EI}.
\end{equation}
Substituting this back:
\begin{equation}\label{eq:app:si_unbiased}
\E[\cD \sim \pi_\theta]{s_i} = \frac{1}{\binom{B-1}{K-1}} \sum_{\cI \subseteq \cU_{-i}, |\cI|=K-1} \mu_{EI} = \mu_{EI}.
\end{equation}

\paragraph{Unbiasedness of the L2O Baseline $b_{-i}^{L2O}$.}
Next, we show that the baseline also targets $\mu_{EI}$.
The L2O baseline is defined as the average of leave-two-out estimators:
\begin{equation}
    b_{-i}^{L2O} := \frac{1}{B-1} \sum_{j \neq i} s_j^{(-i)}.
\end{equation}
Here, $s_j^{(-i)}$ is the EI estimator for sample $j$ computed using indices $\cU_{-ij} = \{1, \dots, B\} \setminus \{i, j\}$.
Crucially, to estimate the same quantity $\mu_{EI}$ (which involves a $1$-vs-$(K-1)$ comparison), $s_j^{(-i)}$ must aggregate over subsets of size $K-1$:
\begin{equation}
    s_j^{(-i)} := \frac{1}{\binom{B-2}{K-1}} \sum_{\cI' \subseteq \cU_{-ij}, |\cI'|=K-1} \left(r_j - \max_{k \in \cI'} r_k\right)_+.
\end{equation}

By the linearity of expectation:
\begin{equation}
    \E[\cD \sim \pi_\theta]{b_{-i}^{L2O}} = \frac{1}{B-1} \sum_{j \neq i} \E[\cD \sim \pi_\theta]{s_j^{(-i)}}.
\end{equation}
Focusing on a single term $\E[\cD \sim \pi_\theta]{s_j^{(-i)}}$:
\begin{equation}
    \E[\cD \sim \pi_\theta]{s_j^{(-i)}} = \frac{1}{\binom{B-2}{K-1}} \sum_{\cI' \subseteq \cU_{-ij}, |\cI'|=K-1} \E[\cD \sim \pi_\theta]{ \left(r_j - \max_{k \in \cI'} r_k\right)_+ }.
\end{equation}
Similar to the proof of the unbiasedness of the EI estimator, for any subset $\cI'$ of size $K-1$ drawn from $\cU_{-ij}$, the set $\{j\} \cup \cI'$ consists of $K$ distinct i.i.d. samples.
Thus, by the i.i.d. assumption, we have:
\begin{equation}
    \E[\cD \sim \pi_\theta]{ \left(r_j - \max_{k \in \cI'} r_k\right)_+ } = \E[a_{1:K}]{ \left(r(a_1) - \max_{k=2,\dots,K} r(a_k)\right)_+ } = \mu_{EI}.
\end{equation}
Summing over all subsets:
\begin{equation}
    \E[\cD \sim \pi_\theta]{s_j^{(-i)}} = \frac{1}{\binom{B-2}{K-1}} \cdot \binom{B-2}{K-1} \cdot \mu_{EI} = \mu_{EI}.
\end{equation}
Finally, averaging over $j \neq i$:
\begin{equation}\label{eq:app:baseline_unbiased}
    \E[\cD \sim \pi_\theta]{b_{-i}^{L2O}} = \frac{1}{B-1} \sum_{j \neq i} \mu_{EI} = \mu_{EI}.
\end{equation}

\paragraph{Exact Centering of the Realized Batch Mean.}
We now show that the batch-average L2O advantage is exactly zero for every realized batch $\cD$.
It suffices to prove that
\begin{equation}
\sum_{i=1}^B b_{-i}^{L2O} = \sum_{i=1}^B s_i.
\end{equation}
Fix $j \in \{1,\dots,B\}$ and consider the average of the leave-two-out EI estimators targeting sample $j$ over all removed indices $i \neq j$:
\begin{align}
\frac{1}{B-1}\sum_{i \neq j} s_j^{(-i)}
&= \frac{1}{B-1}\sum_{i \neq j}
\frac{1}{\binom{B-2}{K-1}}
\sum_{\cI \subseteq \{1,\dots,B\}\setminus\{i,j\},\, |\cI|=K-1}
\left(r_j - \max_{k \in \cI} r_k\right)_+ .
\end{align}
Re-index the inner sum by subsets $\cI \subseteq \cU_{-j} := \{1,\dots,B\}\setminus\{j\}$ of size $K-1$.
For any fixed such subset $\cI$, the term $\left(r_j - \max_{k \in \cI} r_k\right)_+$ appears once for each removed index $i \notin \cI \cup \{j\}$.
There are exactly $B-K$ such choices of $i$.
Therefore,
\begin{align}
\frac{1}{B-1}\sum_{i \neq j} s_j^{(-i)}
&= \frac{B-K}{(B-1)\binom{B-2}{K-1}}
\sum_{\cI \subseteq \cU_{-j},\, |\cI|=K-1}
\left(r_j - \max_{k \in \cI} r_k\right)_+ \\
&= \frac{1}{\binom{B-1}{K-1}}
\sum_{\cI \subseteq \cU_{-j},\, |\cI|=K-1}
\left(r_j - \max_{k \in \cI} r_k\right)_+ \\
&= s_j,
\end{align}
where the second equality uses the combinatorial identity
\begin{equation}
\frac{B-K}{(B-1)\binom{B-2}{K-1}} = \frac{1}{\binom{B-1}{K-1}}.
\end{equation}
Summing this identity over $j$ yields
\begin{align}
\sum_{i=1}^B b_{-i}^{L2O}
&= \frac{1}{B-1}\sum_{i=1}^B \sum_{j \neq i} s_j^{(-i)}
= \sum_{j=1}^B \frac{1}{B-1}\sum_{i \neq j} s_j^{(-i)}
= \sum_{j=1}^B s_j.
\end{align}
Hence,
\begin{equation}
\frac{1}{B}\sum_{i=1}^B \left(s_i - b_{-i}^{L2O}\right) = 0
\end{equation}
for every realized batch $\cD$.

\paragraph{Conclusion.}
Combining Eq. \eqref{eq:app:si_unbiased}, Eq. \eqref{eq:app:baseline_unbiased}, and the exact batch-centering identity above, we conclude that the L2O advantage is unbiased in expectation and exactly centered at the batch level.
\end{proof}

%% file: sec/app/computation.tex
\section{Derivation for Efficient Computation}\label{app:computation}

\subsection{Computation for the EI Estimator}\label{app:computation:batch_ei}
First, we demonstrate how to efficiently compute the EI vector $\mathbf{s} = [s_{(1)}, \dots, s_{(B)}]^\top$ using its rank structure.
Assume rewards are sorted in ascending order: $r_{(1)} \leq \dots \leq r_{(B)}$.
Let $W=\max_{k \in \mathcal{I}} r_{(k)}$ be the maximum reward in the batch.
For a target sample $i$, the term $\bigl(r_{(i)} - W\bigr)_+$ is non-zero only if $W < r_{(i)}$.
The probability that the maximum $W$ equals a specific value $r_{(j)}$ (where $j < i$) follows a hypergeometric distribution depending only on the rank $j$:
$P(W = r_{(j)}) := {\binom{j-1}{K-2}}/{\binom{B-1}{K-1}}, \quad \text{for } j < i$.
Here, the denominator represents the total number of ways to choose $K-1$ samples from $B-1$ candidates.
The numerator corresponds to fixing $r_{(j)}$ as the maximum and choosing the remaining $K-2$ samples from the $j-1$ candidates smaller than $r_{(j)}$.
Consequently, the estimator $s_{(i)}$ can be expressed as a weighted sum of pairwise ReLU differences:
\begin{equation}\label{eq:weighted_relu_difference_ei}
    s_{(i)} = \sum_{j=1}^B P(W = r_{(j)}) \bigl(r_{(i)} - r_{(j)}\bigr)_+,
\end{equation}
This formulation leads directly to a matrix-based computation: multiplying the ReLU-difference matrix by the probability weights.

\begin{theorem}[Efficient Vectorized Computation of Batch EI]\label{thm:batch_ei}
    Let $\mathbf{r} \in \mathbb{R}^B$ be the sorted reward vector.
    Define the \textit{ReLU-difference matrix} $D \in \mathbb{R}^{B \times B}$ as $D_{i,j} = \bigl(r_{(i)} - r_{(j)}\bigr)_+$, and the weight vector $\mathbf{w} \in \mathbb{R}^B$ as $w_j = \binom{j-1}{K-2} / \binom{B-1}{K-1}$.
    The vector of EI estimators $\mathbf{s}$ is given by:
    \begin{equation}\label{eq:ei_estimator}
    \mathbf{s} = D \mathbf{w}.
    \end{equation}
\end{theorem}

We now prove Theorem \ref{thm:batch_ei}, which shows that the EI estimator can be computed efficiently using matrix operations.
\begin{proof}
    Assume $2 \le K \le B$.
    Let the batch rewards be sorted as $r_{(1)} \le \cdots \le r_{(B)}$, and fix a rank index $i \in \{1,\ldots,B\}$.
    By definition,
    \begin{align}
    s_{(i)} := \E[\cI]{ \left(r_{(i)} - \underbrace{\max_{j \in \cI} r_{(j)}}_{:= W}\right)_+ }, \\
            = \sum_{j=1}^B P(W = r_{(j)}) \left(r_{(i)} - r_{(j)}\right)_+,
    \end{align}
    where the probability is for the distribution of the uniformly sampled subset $\cI$ of size $K-1$ from the comparator pool $\cU_{-i} = \{1,\ldots,B\}\setminus\{i\}$.
    Since $(r_{(i)}-r_{(j)})_+=0$ for all $j\ge i$, only ranks $j<i$ can contribute to the expectation.
    Therefore, it suffices to compute $P(W=r_{(j)})$ for $j<i$.

    Fix $j<i$.
    The event $\{W=r_{(j)}\}$ occurs if and only if $j\in\mathcal{I}$ and the remaining $K-2$ indices in $\mathcal{I}$ are chosen from the $j-1$ indices with rewards strictly smaller than $r_{(j)}$.
    The number of such subsets is $\binom{j-1}{K-2}$.
    The total number of comparator subsets is $\binom{B-1}{K-1}$.
    Hence, for $j<i$,
    \begin{equation}
    P(W=r_{(j)}) = \frac{\binom{j-1}{K-2}}{\binom{B-1}{K-1}}.
    \end{equation}
    Using the law of total expectation and the fact that only $j<i$ contributes, we obtain
    \begin{equation}
    s_{(i)} = \sum_{j=1}^B P(W=r_{(j)})\,(r_{(i)}-r_{(j)})_+ = \sum_{j=1}^B w_j\,D_{i,j},
    \end{equation}
    where $D_{i,j}:=(r_{(i)}-r_{(j)})_+$ and $w_j:=\binom{j-1}{K-2}/\binom{B-1}{K-1}$.
    Stacking the identities for all $i$ yields $\mathbf{s}=D\mathbf{w}$.
\end{proof}

\textbf{Complexity Analysis.}
Our method requires $O(B \log B)$ time for sorting and $O(B^2)$ time to construct the pairwise difference matrix and perform the resulting matrix--vector multiplication.
We note that \citet{walder2025pass} derived an algorithm that computes the EI estimator in $O(B \log B + K)$.
In the context of LLM training, the batch size $B$ (typically the group size, e.g., $16$--$64$) is small.
Moreover, our vectorized formulation maps directly to GPU-friendly dense primitives (e.g., broadcasted elementwise operations and matrix--vector products), making it straightforward to leverage GPU parallelism in practice.

\subsection{Computation for the L2O Baseline (proof of Theorem \ref{thm:l2o_baseline_informal})}\label{app:computation:l2o_baseline}
As with the EI estimator, we compute the L2O baseline using matrix multiplication.
To this end, we rewrite $b_{-i}^{\mathrm{L2O}}$ as a weighted sum of ReLU differences:
\begin{align}
    b_{-i}^{\mathrm{L2O}} &= \frac{1}{B-1} \sum_{l \neq i}^B P(W_{-i} = r_{(l)}) \sum_{j \neq i} \bigl(r_{(j)} - r_{(l)}\bigr)_+,
\end{align}
where $W_{-i}$ is the maximum of a subset drawn from the population excluding $i$.
Compared with Eq.~\eqref{eq:weighted_relu_difference_ei}, this form also admits a matrix computation of the weighted sum of ReLU differences.

\begin{theorem}[Efficient Computation of L2O Baseline]\label{thm:l2o_baseline}
    Let $S_l = \sum_{k=1}^B D_{k,l}$ be the column sums of $D$.
    Define the \textit{Leave-One-Out column mean matrix} $\mathcal{M} \in \mathbb{R}^{B \times B}$ as
    $\mathcal{M}_{i,l} = (S_l - D_{i,l})/(B-1)$.
    Then, the vector of L2O baselines $\mathbf{b}^{\mathrm{L2O}} \in \mathbb{R}^B$ is given by the row-wise dot product:
    \begin{equation}
        b_{-i}^{\mathrm{L2O}} = \sum_{l=1}^B \mathcal{M}_{i,l}\, W^{\mathrm{L2O}}_{i,l}, \text{where} \quad W^{\mathrm{L2O}}_{i,l} =
    \begin{cases}
    0 & \text{if } l = i, \\
    \displaystyle \frac{\binom{(l-1) - \mathbb{I}[l > i]}{K-2}}{\binom{B-2}{K-1}} & \text{if } l \neq i.
    \end{cases}
    \end{equation}
    Here, the indicator $\mathbb{I}[l > i]$ adjusts the rank of $r_{(l)}$ after removing index $i$ from the sorted batch.

\end{theorem}
We now prove Theorem \ref{thm:l2o_baseline}, which shows that the L2O baseline can be computed efficiently using matrix operations.

\begin{proof}
    Assume $2 \le K \le B-1$.
    Let $r_{(1)} \le \cdots \le r_{(B)}$ be the sorted rewards.
    Define $D\in\R^{B\times B}$ by $D_{j,l} := \bigl(r_{(j)}-r_{(l)}\bigr)_+$.
    Fix an index $i\in\{1,\ldots,B\}$.
    Recall that
    \begin{equation}
        b_{-i}^{\mathrm{L2O}} := \frac{1}{B-1}\sum_{j\neq i} s_j^{(-i)}.
    \end{equation}
    For each $j\neq i$, define $s_j^{(-i)}$ by sampling $\mathcal{I}'$ uniformly without replacement among all subsets of size $K-1$ from $\{1,\ldots,B\}\setminus\{i,j\}$ and setting $W_{-i} := \max_{k\in\mathcal{I}'} r_{(k)}$.
    Then
    \begin{equation}\label{eq:app:l2o_total_expectation}
        s_j^{(-i)} = \E[\mathcal{I}']{ \bigl(r_{(j)} - W_{-i}\bigr)_+ } = \sum_{l\in\{1,\ldots,B\}\setminus\{i,j\}}\!P\!\left(W_{-i} = r_{(l)}\right)\,\bigl(r_{(j)}-r_{(l)}\bigr)_+.
    \end{equation}
    Note that $D_{j,l}=(r_{(j)}-r_{(l)})_+=0$ for all $l\ge j$.
    For $l \ge j$, we have $D_{j,l} = (r_{(j)}-r_{(l)})_+ = 0$, hence these terms do not contribute regardless of the value assigned to $W^{\mathrm{L2O}}_{i,l}$.
    Therefore, it suffices to compute $P(W_{-i}=r_{(l)})$ for $l<j$.

    Fix $l<j$ with $l\neq i$.
    The event $\{W_{-i}=r_{(l)}\}$ occurs if and only if $l\in\mathcal{I}'$ and the remaining $K-2$ indices in $\mathcal{I}'$ are chosen from indices with rank strictly below $l$.
    Among $\{1,\ldots,l-1\}$, the only index that may be excluded from the sampling pool $\{1,\ldots,B\}\setminus\{i,j\}$ is $i$ when $i<l$.
    Since $l<j$, excluding $j$ does not remove any index from $\{1,\ldots,l-1\}$.
    Thus, the number of available indices strictly below $l$ in the pool is $(l-1)-\mathbb{I}[l>i]$.
    Hence, the number of valid subsets $\mathcal{I}'$ of size $K-1$ for which $W_{-i}=r_{(l)}$ is $\binom{(l-1)-\mathbb{I}[l>i]}{K-2}$.
    The total number of possible subsets $\mathcal{I}'$ is $\binom{B-2}{K-1}$.
    Therefore, for any $l<j$ with $l\neq i$,
    \begin{equation}\label{eq:app:l2o_prob}
        P\!\left(W_{-i}=r_{(l)}\right)
        = \frac{\binom{(l-1)-\mathbb{I}[l>i]}{K-2}}{\binom{B-2}{K-1}}
        = W^{\mathrm{L2O}}_{i,l}.
    \end{equation}
    Crucially, the right-hand side depends on $i$ and $l$ but not on $j$, as long as $l<j$.

    Substituting \eqref{eq:app:l2o_prob} into \eqref{eq:app:l2o_total_expectation} and using $D_{j,l}=(r_{(j)}-r_{(l)})_+$ yields
    \begin{equation}\label{eq:app:l2o_sj_weighted}
        s_j^{(-i)} = \sum_{l=1}^B W^{\mathrm{L2O}}_{i,l}\,D_{j,l}.
    \end{equation}
    Averaging over $j\neq i$ and exchanging the finite sums gives
    \begin{equation}\label{eq:app:l2o_exchange}
        b_{-i}^{\mathrm{L2O}}
        = \frac{1}{B-1}\sum_{j\neq i}\sum_{l=1}^B W^{\mathrm{L2O}}_{i,l}\,D_{j,l}
        = \sum_{l=1}^B W^{\mathrm{L2O}}_{i,l}\,\frac{1}{B-1}\sum_{j\neq i} D_{j,l}.
    \end{equation}
    Let $S_l := \sum_{k=1}^B D_{k,l}$.
    Then
    \begin{equation}
        \frac{1}{B-1}\sum_{j\neq i} D_{j,l}
        = \frac{S_l - D_{i,l}}{B-1}
        =: \mathcal{M}_{i,l}.
    \end{equation}
    Substituting into \eqref{eq:app:l2o_exchange} yields
    \begin{equation}
        b_{-i}^{\mathrm{L2O}} = \sum_{l=1}^B \mathcal{M}_{i,l}\,W^{\mathrm{L2O}}_{i,l}.
    \end{equation}
    This holds for each $i\in\{1,\ldots,B\}$, which completes the proof.
\end{proof}

%% file: sec/app/unified.tex
\section{Unified View of Estimators}\label{app:unified}

In this appendix, we provide a unified view of our proposed estimators and those of \citet{walder2025pass} from the perspective of marginal statistics.

\subsection{Definitions of Marginal Statistics}
Let $\mathcal{B} = \{1, \dots, B\}$ be the set of indices for the batch of samples.
We assume a fixed batch size $B$ and a subset size $K$ such that $2 \le K \le B-1$.
For any subset $S \subseteq \mathcal{B}$, let $M(S) := \max_{j \in S} r_j$.
All expectations below are finite averages over subsets, equivalent to sampling uniformly without replacement.

We formally define the two key statistics $u_i$ and $v_i$ for a specific index $i \in \mathcal{B}$:
\begin{itemize}
    \item \textbf{Conditional Expected Max $u_i$}: The expected maximum of a size-$K$ group conditioned on containing $i$:
    \begin{equation}
        u_i := \frac{1}{\binom{B-1}{K-1}} \sum_{S \subseteq \mathcal{B} \setminus \{i\}, |S|=K-1} \max(r_i, M(S)).
    \end{equation}
    \item \textbf{Leave-One-Out Expected Max $v_i$}: The expected maximum of a size-$K$ group drawn from the pool excluding $i$:
    \begin{equation}
        v_i := \frac{1}{\binom{B-1}{K}} \sum_{T \subseteq \mathcal{B} \setminus \{i\}, |T|=K} M(T).
    \end{equation}
\end{itemize}

The quantity $\gamma(u_i - v_i)$, where $\gamma = K/B$, represents the exact marginal contribution of sample $i$ to the max@K objective under finite without-replacement sampling.
We first record the finite-batch centering proof for this canonical signal; the equivalence between EI+L2O and this canonical signal is proved in the following subsection.

\paragraph{Centering of the canonical finite-batch advantage.}
The leave-one-out term $v_i$ depends only on rewards in $\mathcal{B}\setminus\{i\}$, so it is a valid leave-one-out baseline for response $i$ and preserves policy-gradient unbiasedness.
It remains to prove that the realized batch average of $u_i-v_i$ is exactly zero.
Summing the definitions over all $i$, every size-$K$ subset $V\subseteq\mathcal{B}$ contributes to $u_i$ exactly when $i\in V$.
Since $|V|=K$, each $M(V)$ is counted exactly $K$ times:
\begin{equation}
    \sum_{i=1}^B u_i
    =
    \frac{K}{\binom{B-1}{K-1}}
    \sum_{\substack{V\subseteq\mathcal{B}\\ |V|=K}} M(V).
\end{equation}
Similarly, a size-$K$ subset $V$ contributes to $v_i$ exactly when $i\notin V$.
There are $B-K$ such indices, so
\begin{equation}
    \sum_{i=1}^B v_i
    =
    \frac{B-K}{\binom{B-1}{K}}
    \sum_{\substack{V\subseteq\mathcal{B}\\ |V|=K}} M(V).
\end{equation}
The coefficients are equal because
\begin{equation}
    \frac{K}{\binom{B-1}{K-1}}
    =
    \frac{B-K}{\binom{B-1}{K}}.
\end{equation}
Therefore $\sum_i u_i=\sum_i v_i$, and hence
\begin{equation}
    \frac{1}{B}\sum_{i=1}^B (u_i-v_i)=0
\end{equation}
for every realized batch.
This establishes exact centering of the canonical finite-batch advantage.

\subsection{Relationship Between EI-L2O and Marginal Statistics} \label{app:unified:relationship_between_ei_l2o_and_marginal_statistics}
We now show that our bias-corrected signal $\gamma(s_i - b_i)$ is identical to the marginal contribution $\gamma(u_i - v_i)$.
First, we explicitly define the finite-batch Expected Improvement $s_i$ and the L2O baseline $b_i$:
\begin{itemize}
    \item \textbf{Expected Improvement $s_i$}: The direct finite-sample version of $s_i = \mathbb{E}[(r_i - W)_+]$ where $W$ is the maximum of $K-1$ comparators:
    \begin{equation}
        s_i := \frac{1}{\binom{B-1}{K-1}} \sum_{S \subseteq \mathcal{B} \setminus \{i\}, |S|=K-1} (r_i - M(S))_+.
    \end{equation}
    \item \textbf{L2O Baseline $b_i$}: The average of leave-two-out EI values:
    \begin{equation}
        b_i := \frac{1}{B-1} \sum_{j \neq i} s_j^{(-i)},
    \end{equation}
    where $s_j^{(-i)}$ is the EI of $j$ computed over the pool $\mathcal{B} \setminus \{i, j\}$.
\end{itemize}

Using the identity $\max(r_i, M(S)) = M(S) + (r_i - M(S))_+$, we decompose $u_i$ as:
\begin{equation}
    u_i = w_i + s_i,
\end{equation}
where $w_i$ is the expected comparator-max when sampling $K-1$ items from the pool excluding $i$.
We now show that the L2O baseline removes exactly the gap between the comparator-max baseline $w_i$ and the leave-one-out max@$K$ baseline $v_i$.
Fix $i$ and expand
\begin{align}
b_i
&=
\frac{1}{B-1}\sum_{j\ne i}
\frac{1}{\binom{B-2}{K-1}}
\sum_{\substack{S\subseteq\mathcal{B}\setminus\{i,j\}\\ |S|=K-1}}
\left(M(S\cup\{j\})-M(S)\right).
\end{align}
For the first term, every size-$K$ subset $T\subseteq\mathcal{B}\setminus\{i\}$ appears exactly $K$ times, once for each possible choice of $j\in T$.
Therefore,
\begin{equation}
\frac{1}{(B-1)\binom{B-2}{K-1}}
\sum_{j\ne i}
\sum_{\substack{S\subseteq\mathcal{B}\setminus\{i,j\}\\ |S|=K-1}}
M(S\cup\{j\})
=
\frac{1}{\binom{B-1}{K}}
\sum_{\substack{T\subseteq\mathcal{B}\setminus\{i\}\\ |T|=K}}
M(T)
=
v_i.
\end{equation}
For the second term, every size-$(K-1)$ subset $S\subseteq\mathcal{B}\setminus\{i\}$ appears exactly $B-K$ times, once for each $j\in\mathcal{B}\setminus(\{i\}\cup S)$.
Thus,
\begin{equation}
\frac{1}{(B-1)\binom{B-2}{K-1}}
\sum_{j\ne i}
\sum_{\substack{S\subseteq\mathcal{B}\setminus\{i,j\}\\ |S|=K-1}}
M(S)
=
\frac{1}{\binom{B-1}{K-1}}
\sum_{\substack{S\subseteq\mathcal{B}\setminus\{i\}\\ |S|=K-1}}
M(S)
=
w_i.
\end{equation}
Hence $b_i=v_i-w_i$.
Combining this with $s_i=u_i-w_i$ yields
\begin{equation}
    s_i - b_i = (u_i - w_i) - (v_i - w_i) = u_i - v_i.
\end{equation}
Thus, the scaled signal $\gamma(s_i - b_i)$ exactly recovers the marginal contribution $\gamma(u_i - v_i)$.

\subsection{L2O for Arbitrary Comparator-Only Baselines}
\label{app:unified:arbitrary_comparator_baselines}

The previous subsection treats EI, which corresponds to subtracting the comparator-max baseline $C(S)=M(S)$ before applying L2O.
The same cancellation holds for any baseline that depends only on the comparator set.

\begin{proposition}[L2O cancels arbitrary comparator-only baselines]
\label{prop:unified:l2o_arbitrary_comparator_baseline}
Fix $2\le K\le B-1$ and a realized batch $\mathcal{B}$.
Let $C(S)$ be any scalar-valued function of a comparator set $S\subseteq\mathcal{B}$ with $|S|=K-1$.
For each $i\in\mathcal{B}$, define the comparator-baselined group signal
\begin{equation}
    \alpha_i^C
    :=
    \frac{1}{\binom{B-1}{K-1}}
    \sum_{\substack{S\subseteq\mathcal{B}\setminus\{i\}\\ |S|=K-1}}
    \left(M(S\cup\{i\})-C(S)\right).
\end{equation}
Define its L2O baseline by recomputing the same signal for each $j\ne i$ after removing $i$:
\begin{equation}
    b_i^C
    :=
    \frac{1}{B-1}\sum_{j\ne i}\alpha_j^{C,(-i)},
\end{equation}
where
\begin{equation}
    \alpha_j^{C,(-i)}
    :=
    \frac{1}{\binom{B-2}{K-1}}
    \sum_{\substack{S\subseteq\mathcal{B}\setminus\{i,j\}\\ |S|=K-1}}
    \left(M(S\cup\{j\})-C(S)\right).
\end{equation}
Then
\begin{equation}
    \alpha_i^C-b_i^C=u_i-v_i
\end{equation}
for every $i\in\mathcal{B}$.
\end{proposition}

\begin{proof}
Define the averaged comparator-only baseline for response $i$ by
\begin{equation}
    c_i
    :=
    \frac{1}{\binom{B-1}{K-1}}
    \sum_{\substack{S\subseteq\mathcal{B}\setminus\{i\}\\ |S|=K-1}}
    C(S).
\end{equation}
By the definition of $u_i$, the starting signal decomposes as
\begin{equation}
    \alpha_i^C = u_i-c_i.
\end{equation}
It remains to compute the L2O baseline $b_i^C$.
Expanding its definition gives
\begin{align}
b_i^C
&=
\frac{1}{B-1}\sum_{j\ne i}
\frac{1}{\binom{B-2}{K-1}}
\sum_{\substack{S\subseteq\mathcal{B}\setminus\{i,j\}\\ |S|=K-1}}
\left(M(S\cup\{j\})-C(S)\right).
\end{align}
We handle the $M$ part and the $C$ part separately.
For the $M$ part, every size-$K$ subset $T\subseteq\mathcal{B}\setminus\{i\}$ appears exactly $K$ times, once for each possible choice of $j\in T$.
Therefore,
\begin{align}
&\frac{1}{(B-1)\binom{B-2}{K-1}}
\sum_{j\ne i}
\sum_{\substack{S\subseteq\mathcal{B}\setminus\{i,j\}\\ |S|=K-1}}
M(S\cup\{j\}) \\
&\qquad =
\frac{K}{(B-1)\binom{B-2}{K-1}}
\sum_{\substack{T\subseteq\mathcal{B}\setminus\{i\}\\ |T|=K}}
M(T) \\
&\qquad =
\frac{1}{\binom{B-1}{K}}
\sum_{\substack{T\subseteq\mathcal{B}\setminus\{i\}\\ |T|=K}}
M(T)
=
v_i.
\end{align}
For the comparator-only baseline part, every size-$(K-1)$ subset $S\subseteq\mathcal{B}\setminus\{i\}$ appears exactly $B-K$ times, once for each choice of $j\in\mathcal{B}\setminus(\{i\}\cup S)$.
Thus,
\begin{align}
&\frac{1}{(B-1)\binom{B-2}{K-1}}
\sum_{j\ne i}
\sum_{\substack{S\subseteq\mathcal{B}\setminus\{i,j\}\\ |S|=K-1}}
C(S) \\
&\qquad =
\frac{B-K}{(B-1)\binom{B-2}{K-1}}
\sum_{\substack{S\subseteq\mathcal{B}\setminus\{i\}\\ |S|=K-1}}
C(S) \\
&\qquad =
\frac{1}{\binom{B-1}{K-1}}
\sum_{\substack{S\subseteq\mathcal{B}\setminus\{i\}\\ |S|=K-1}}
C(S)
=
c_i.
\end{align}
Combining the two parts yields
\begin{equation}
    b_i^C=v_i-c_i.
\end{equation}
Therefore,
\begin{equation}
    \alpha_i^C-b_i^C=(u_i-c_i)-(v_i-c_i)=u_i-v_i.
\end{equation}
\end{proof}

This proposition includes EI as the special case $C(S)=M(S)$, for which $\alpha_i^C=s_i$, and the raw conditional max signal as the special case $C(S)=0$, for which $\alpha_i^C=u_i$.
The comparator-only condition is essential: if the added term depends on the target response itself, then the leave-out property and the cancellation above need not hold.

\subsection{Comparison with \citet{walder2025pass}'s estimators} \label{app:unified:comparison_with_walder2025pass}
\citet{walder2025pass} provide two primary advantage estimators: \texttt{sloo\_minus\_one} (their Eq. 33) and \texttt{sloo} (their Eq. 29).
Here, we formally show the equivalence between EI and their \texttt{sloo\_minus\_one} estimator (max@K - max@(K-1)).
We then describe how the other \texttt{sloo} estimator is biased differently from \texttt{sloo\_minus\_one}.

\subsubsection*{$\texttt{sloo\_minus\_one}$ and EI-only Signal}
The \texttt{sloo\_minus\_one} variant is defined as the difference between the max of a size-$K$ subset and the max of the same subset excluding index $i$.
\begin{equation}
    s_i^{loo-1} = \frac{1}{\binom{B}{K}} \sum_{S \subseteq \mathcal{B} \setminus \{i\}, |S|=K-1} \left( \max(r_i, M(S)) - M(S) \right).
\end{equation}
Using the fundamental identity relating the maximum to the ReLU (Positive Part) function, $\max(a, b) = b + (a - b)_+$, the term inside the summation becomes:
\begin{equation}
    \max(r_i, M(S)) - M(S) = (r_i - M(S))_+.
\end{equation}
Substituting this back into the estimator and multiplying by $\binom{B-1}{K-1} / \binom{B-1}{K-1}$ to align with the definition of Expected Improvement ($s_i$):
\begin{align}
    s_i^{loo-1} &= \frac{\binom{B-1}{K-1}}{\binom{B}{K}} \left[ \frac{1}{\binom{B-1}{K-1}} \sum_{S \subseteq \mathcal{B} \setminus \{i\}, |S|=K-1} (r_i - M(S))_+ \right] \\
    &= \frac{K}{B} s_i = \gamma s_i.
\end{align}
This proves that their \texttt{sloo\_minus\_one} is exactly equivalent to our EI-only signal scaled by $\gamma$.

\subsubsection*{$\texttt{sloo}$ and Rescaled Marginal Signal}
\citet{walder2025pass}'s \texttt{sloo} variant is defined as:
\begin{equation}
    s_i^{\text{loo}} = S(i, K, \mathcal{B}) - \frac{1}{B-1} \sum_{j \in \mathcal{B} \setminus \{i\}} S(j, K, \mathcal{B} \setminus \{i\}).
\end{equation}
\citet{walder2025pass} define $S(i, K, \mathcal{U})$ as a normalized sum over $K$-subsets containing index $i$:
\begin{equation}
    S(i, K, \mathcal{U}) := \frac{1}{\binom{|\mathcal{U}|}{K}} \sum_{I \subseteq \mathcal{U}, |I|=K, i \in I} \max_{t \in I} r_t.
\end{equation}

Based on this definition, we derive the connection to our $u_i$ and $v_i$:
\begin{enumerate}
    \item \textbf{Connection to $u_i$}: Every $K$-subset $I \subseteq \mathcal{B}$ with $i \in I$ can be written as $\{i\} \cup S$ where $|S|=K-1$. Thus:
    \begin{equation}
        S(i, K, \mathcal{B}) = \frac{\binom{B-1}{K-1}}{\binom{B}{K}} u_i = \frac{K}{B} u_i.
    \end{equation}
    \item \textbf{Connection to $v_i$}: Consider the sum over the reduced pool $U = \mathcal{B} \setminus \{i\}$:
    \begin{equation}
        \sum_{j \in U} S(j, K, U) = \frac{1}{\binom{B-1}{K}} \sum_{I \subseteq U, |I|=K} \left( \sum_{j \in I} 1 \right) M(I) = K v_i,
    \end{equation}
    since each $K$-subset $I$ contains exactly $K$ indices $j$.
    Dividing by $B-1$ yields:
    \begin{equation}
        \frac{1}{B-1} \sum_{j \neq i} S(j, K, \mathcal{B} \setminus \{i\}) = \frac{K}{B-1} v_i.
    \end{equation}
\end{enumerate}

Substituting these into the definition of Eq. (29):
\begin{equation}
    s^{\text{loo}}_i = \frac{K}{B} u_i - \frac{K}{B-1} v_i = \gamma \left( u_i - \frac{B}{B-1} v_i \right).
\end{equation}

Comparing this to our signal $\gamma(u_i - v_i)$, we see a discrepancy in the coefficient of $v_i$.
Our method uses $u_i - v_i$, whereas \citet{walder2025pass} implicitly use $u_i - \frac{B}{B-1} v_i$.
Since our EI + L2O signal $u_i - v_i$ is the unbiased advantage estimator, their $s_i^{\text{loo}}$ is biased at finite $B$: it underestimates the advantage by $\frac{\gamma}{B-1}v_i$.
Furthermore, \citet{walder2025pass} did not evaluate $\texttt{sloo}$ experimentally, and we likewise exclude it because there is no clear justification for using a biased estimator.
In summary, their $\texttt{sloo\_minus\_one}$ (max@K - max@(K-1)) estimator tends to overestimate the advantage, whereas $\texttt{sloo}$ underestimates it at finite batch sizes. This contrast highlights the theoretical motivation for our L2O-based approach.

\subsection{Comparison with the Analytical \texorpdfstring{$\mathrm{pass@}K$}{pass@K} Estimator of \citet{chen2025pass}} \label{app:unified:comparison_with_chen2025pass}

We next relate the analytical binary $\mathrm{pass@}K$ estimator of \citet{chen2025pass} to the marginal statistics $u_i$ and $v_i$.
This comparison applies to binary rewards, $r_i\in\{0,1\}$.
Let $N_{\mathrm{neg}}$ be the number of negative responses in the batch.
For binary rewards, the all-subsets mean $\bar m$ is the empirical $\mathrm{pass@}K$ value:
\begin{equation}
    \bar m
    =
    1
    -
    \frac{\binom{N_{\mathrm{neg}}}{K}}{\binom{B}{K}}.
\label{eq:chen_mbar_binary}
\end{equation}
This is the group-level mean used in the analytical derivation of \citet{chen2025pass}.

We show that the raw analytical numerator of \citet{chen2025pass}, before division by the group standard deviation, is exactly
\begin{equation}
    u_i-\bar m.
\label{eq:chen_raw_ui_mbar}
\end{equation}
There are two cases.

\paragraph{Positive response.}
If $r_i=1$, then every size-$K$ subset containing $i$ has maximum reward one.
Therefore,
\begin{equation}
    u_i=1.
\label{eq:chen_positive_ui}
\end{equation}
The raw analytical numerator for a positive response is the positive group reward minus the group mean:
\begin{equation}
    A^{\mathrm{Chen,raw}}_i
    =
    1-\bar m
    =
    u_i-\bar m.
\label{eq:chen_positive_raw}
\end{equation}

\paragraph{Negative response.}
If $r_i=0$, then a size-$K$ subset containing $i$ is positive if and only if at least one of the other $K-1$ responses is positive.
Hence
\begin{equation}
    u_i
    =
    1
    -
    \frac{\binom{N_{\mathrm{neg}}-1}{K-1}}{\binom{B-1}{K-1}}.
\label{eq:chen_negative_ui}
\end{equation}
The raw analytical numerator for a negative response is the probability that a group containing this response is positive, minus the group mean:
\begin{equation}
    A^{\mathrm{Chen,raw}}_i
    =
    1
    -
    \frac{\binom{N_{\mathrm{neg}}-1}{K-1}}{\binom{B-1}{K-1}}
    -
    \bar m
    =
    u_i-\bar m.
\label{eq:chen_negative_raw}
\end{equation}
Thus, for both positive and negative responses,
\begin{equation}
    A^{\mathrm{Chen,raw}}_i
    =
    u_i-\bar m.
\label{eq:chen_raw_equals_centered_ui}
\end{equation}

We next relate $u_i-\bar m$ to the canonical signal $u_i-v_i$.
Partitioning all size-$K$ subsets of $\mathcal{B}$ into those containing $i$ and those not, and using $\binom{B-1}{K-1}/\binom{B}{K}=K/B$ and $\binom{B-1}{K}/\binom{B}{K}=(B-K)/B$, the all-subsets mean satisfies
\begin{equation}
    \bar m
    =
    \frac{K}{B}\, u_i + \frac{B-K}{B}\, v_i.
\label{eq:mbar_decomposition}
\end{equation}
Subtracting both sides from $u_i$ yields
\begin{equation}
    u_i-\bar m
    =
    \frac{B-K}{B}\,(u_i-v_i).
\label{eq:centered_ui_equivalence}
\end{equation}

Using Eq.~\eqref{eq:centered_ui_equivalence}, we obtain
\begin{equation}
    A^{\mathrm{Chen,raw}}_i
    =
    u_i-\bar m
    =
    \frac{B-K}{B}(u_i-v_i).
\label{eq:chen_raw_equals_maxpo_direction}
\end{equation}
Therefore, in the binary setting, the raw analytical estimator of \citet{chen2025pass} has the same centered direction as MaxPO up to a fixed scale.

The estimator used in \citet{chen2025pass} further divides this numerator by a group standard deviation $\sigma_{\mathrm{group}}$.
Thus,
\begin{equation}
    A^{\mathrm{Chen,std}}_i
    =
    \frac{u_i-\bar m}{\sigma_{\mathrm{group}}}.
\label{eq:chen_std_signal}
\end{equation}
This normalization may be useful as an optimization heuristic, but it introduces a random batch-dependent scale.
Consequently, the numerator is proportional to the MaxPO marginal signal, while the standard-deviation-normalized coefficient is not the exact no-std policy-gradient estimator derived from $u_i-v_i$.

\subsection{Comparison with the All-Subsets \texorpdfstring{$\mathrm{max@}K$}{max@K} Transform of \citet{bagirov2025best}} \label{app:unified:comparison_with_bagirov2025best}

We now compare MaxPO with the on-policy all-subsets transformation of \citet{bagirov2025best}.
The derivation in this subsection is on-policy and does not include the probability-ratio correction used in their off-policy estimator.

Let $\tilde r_i$ (Eq. (9) of \citet{bagirov2025best}) denote the raw all-subsets transformed reward assigned to sample $i$:
\begin{equation}
    \tilde r_i
    :=
    \frac{1}{\binom{B}{K}}
    \sum_{\substack{I\subseteq\mathcal{B}\\ |I|=K,\ i\in I}}
    M(I).
\label{eq:bagirov_raw_transform}
\end{equation}
Every size-$K$ subset $I$ containing $i$ can be written as $I=S\cup\{i\}$, where $S\subseteq\mathcal{B}\setminus\{i\}$ and $|S|=K-1$.
Therefore,
\begin{align}
    \tilde r_i
    &=
    \frac{1}{\binom{B}{K}}
    \sum_{\substack{S\subseteq\mathcal{B}\setminus\{i\}\\ |S|=K-1}}
    M(S\cup\{i\})
    \nonumber\\
    &=
    \frac{\binom{B-1}{K-1}}{\binom{B}{K}}
    u_i
    =
    \frac{K}{B}u_i.
\label{eq:bagirov_raw_equals_ui}
\end{align}
Thus, the raw on-policy transformation of \citet{bagirov2025best} is an uncentered all-subsets marginal signal.

Next, consider ordinary mean-centering of the transformed rewards.
Since
\begin{equation}
    \frac{1}{B}\sum_{i=1}^{B}u_i
    =
    \bar m,
\label{eq:mean_ui_equals_mbar}
\end{equation}
we have
\begin{equation}
    \bar{\tilde r}
    :=
    \frac{1}{B}\sum_{i=1}^{B}\tilde r_i
    =
    \frac{K}{B}\bar m.
\label{eq:bagirov_mean_transform}
\end{equation}
Therefore,
\begin{align}
    \tilde r_i-\bar{\tilde r}
    &=
    \frac{K}{B}(u_i-\bar m)
    \nonumber\\
    &=
    \frac{K}{B}\cdot \frac{B-K}{B}(u_i-v_i)
    \nonumber\\
    &=
    \frac{K(B-K)}{B^2}(u_i-v_i).
\label{eq:bagirov_mean_centered_equals_maxpo}
\end{align}
Thus, the mean-centered on-policy all-subsets transform of \citet{bagirov2025best} recovers the same centered direction as MaxPO up to a fixed scale.

If one instead applies z-score normalization to the transformed rewards, the numerator remains proportional to $u_i-v_i$, but the denominator introduces a random batch-dependent scale:
\begin{equation}
    \frac{\tilde r_i-\bar{\tilde r}}{\mathrm{std}(\tilde r)}
    =
    \frac{\frac{K(B-K)}{B^2}(u_i-v_i)}{\mathrm{std}(\tilde r)}.
\label{eq:bagirov_zscore_signal}
\end{equation}
As with the standard-deviation-normalized analytical $\mathrm{pass@}K$ estimator, this changes the exact no-std policy-gradient coefficient.

Finally, the off-policy estimator of \citet{bagirov2025best} is not simply another baseline choice for $u_i$ or $v_i$.
It introduces probability ratios for samples drawn from an older policy and then applies a first-order approximation to the product of ratios.
This produces a different approximate estimator.
The equivalences above therefore apply only to the on-policy all-subsets transformation and its mean-centered or z-score-normalized variants.

\paragraph{Clarifying ``BoN mean'' versus mean-centered EI.}
A possible source of ambiguity is the phrase ``BoN mean'' in \citet{bagirov2025best}.
One possible reading is that BoN mean starts from the EI or LOO-1 signal and then applies the usual GRPO-style mean and standard-deviation normalization.
Under this reading, the centered signal would be
\begin{equation}
    A_i^{\mathrm{meanEI}}
    :=
    s_i-\bar s,
    \qquad
    \bar s
    :=
    \frac{1}{B}\sum_{j=1}^{B}s_j.
\label{eq:mean_centered_ei}
\end{equation}
This is not MaxPO.
Using $s_i=u_i-w_i$, we have
\begin{equation}
    A_i^{\mathrm{meanEI}}
    =
    (u_i-w_i)-(\bar u-\bar w),
\label{eq:mean_centered_ei_expanded}
\end{equation}
which does not simplify to $u_i-v_i$ in general.
Moreover, $\bar s$ can depend on $r_i$, because $r_i$ may appear as a comparator inside $s_j$ for $j\ne i$.
Thus, mean-centering EI is not the L2O correction and does not generally satisfy the leave-out condition required by the policy-gradient unbiasedness argument.

The all-subsets transform in Eq.~\eqref{eq:bagirov_raw_transform} is a different vector.
It is proportional to $u_i$, not to the EI signal $s_i=u_i-w_i$:
\begin{equation}
    \tilde r_i
    =
    \frac{K}{B}u_i.
\end{equation}
Therefore, centering $\tilde r_i$ is fundamentally different from centering $s_i$.
As shown in Eq.~\eqref{eq:bagirov_mean_centered_equals_maxpo},
\begin{equation}
    \tilde r_i-\bar{\tilde r}
    =
    \frac{K(B-K)}{B^2}(u_i-v_i),
\end{equation}
so the mean-centered all-subsets transform is proportional to the MaxPO direction.
In short,
\begin{equation}
    s_i-\bar s
    \not\equiv
    u_i-v_i,
    \qquad
    \tilde r_i-\bar{\tilde r}
    \propto
    u_i-v_i.
\label{eq:bagirov_centering_distinction}
\end{equation}
The distinction is entirely about which vector is being centered.

%% file: sec/app/toy_experiments.tex
\section{Toy Experiments}\label{app:toy_experiments}
In this section, we explain the detailed settings and additional results for the toy experiments we presented in Sec.~\ref{sec:toy_experiment}.

\subsection{Bandits}\label{app:toy_experiments:bandits}
\subsubsection{Setting}\label{app:toy_experiments:bandits:setting}

We generate random bandit instances to rigorously evaluate the statistical properties---bias and variance---of the gradient estimators.
For each problem instance (seed), we sample the reward vector $\mathbf{r} \in \mathbb{R}^{|\mathcal{A}|}$ and the policy logit vector $\boldsymbol{\theta} \in \mathbb{R}^{|\mathcal{A}|}$ from a standard normal distribution:
\begin{equation}
    \boldsymbol{\theta} \sim \mathcal{N}(0, I), \quad \mathbf{r} \sim \mathcal{N}(0, I).
\end{equation}
The policy is defined as a softmax distribution over actions, $\pi_\theta(a) = \mathrm{softmax}(\boldsymbol{\theta})_a$.
Unless otherwise stated, we fix the max@K parameter to $K=2$ and the batch size to $B=8$.
We construct the PG estimator $\hat{g}$ using EI: $s_i$, EI+L2O: $s_i - b^{\mathrm{L2O}}_{-i}$, and EI+L1O: $s_i - b^{\mathrm{L1O}}_{-i}$.
For example, the EI estimator for $\nabla J^K(\theta)$ takes the form $\hat{g} = \frac{K}{B} \sum_{i=1}^B s_i \nabla_{\theta} \log \pi_\theta(a_i)$.

\paragraph{Ground Truth and Estimators.}
For a fixed problem instance specified by $(\boldsymbol{\theta},\mathbf{r})$, we compute the ground-truth max@K policy gradient
$g_{\mathrm{true}} := \nabla J^K(\boldsymbol{\theta})$
analytically via the closed-form derivative of the expected improvement (Proposition~1 of \citet{nishimori2026emergence}).
Each gradient estimator $\hat g$ is computed from a group of $B$ i.i.d.\ actions sampled from $\pi_{\boldsymbol{\theta}}$.

\paragraph{Estimation Error (``Bias'') Protocol.}
For each fixed instance $(\boldsymbol{\theta},\mathbf{r})$, we generate $N$ independent gradient estimates
$\{\hat g_j\}_{j=1}^N$.
We measure the (finite-sample) estimation error of the Monte Carlo mean relative to the ground truth:
\begin{equation}
    \mathrm{Err}(\boldsymbol{\theta},\mathbf{r})
    := \left\| \bar g_N - g_{\mathrm{true}} \right\|_2,
    \qquad
    \bar g_N := \frac{1}{N}\sum_{j=1}^N \hat g_j.
\end{equation}

\paragraph{Empirical Total Variance Protocol.}
To quantify estimator stability, we measure the empirical total variance, defined as the trace of the sample covariance:
\begin{equation}
    \widehat{\mathrm{TV}}_N(\boldsymbol{\theta},\mathbf{r})
    := \mathrm{Tr}\!\left(\frac{1}{N}\sum_{j=1}^N(\hat g_j-\bar g_N)(\hat g_j-\bar g_N)^\top\right)
    = \frac{1}{N}\sum_{j=1}^N \|\hat g_j-\bar g_N\|_2^2.
\end{equation}
Equivalently, using the second-moment identity,
\begin{equation}
    \widehat{\mathrm{TV}}_N(\boldsymbol{\theta},\mathbf{r})
    = \frac{1}{N}\sum_{j=1}^N \|\hat g_j\|_2^2 - \|\bar g_N\|_2^2.
\end{equation}
Unless stated otherwise, we use $N=10^5$ to estimate $\widehat{\mathrm{TV}}_N$.

The total variance of score-function gradients can scale with the action-space dimensionality.
To assess variance reduction in high-dimensional regimes (relevant to LLM settings), we report $\widehat{\mathrm{TV}}_N$ across varying numbers of actions.

\paragraph{Aggregation Across Seeds.}
We repeat the above procedure over $L=100$ independent random seeds, each generating a new pair $(\boldsymbol{\theta},\mathbf{r})$,
and report the mean and standard error of $\mathrm{Err}(\boldsymbol{\theta},\mathbf{r})$ and $\widehat{\mathrm{TV}}_N(\boldsymbol{\theta},\mathbf{r})$ across seeds.

\paragraph{Sweep Configurations.}
For the bias plot, we vary the number of batches $N \in \{10^3, 10^4, 10^5, 10^6\}$ to assess whether the estimation error decreases with $N$, fixing $B=8$ and $K=2$.
To examine variance, we vary the action space size $|\cA| \in \{10, 50, 100, 1000\}$ (fixing $K=2$, $B=8$) and the comparator size $K \in \{2,3,4,5,6\}$ (fixing $B=8$, $|\cA|=100$).
Larger $|\cA|$ corresponds to LLM-relevant regimes; varying $K$ probes how the available comparator budget affects the L2O baseline.

\paragraph{Comparison Baselines.}
We compare three estimators: EI ($s_i$), EI+L2O ($s_i - b^{\mathrm{L2O}}_{-i}$), and EI+L1O ($s_i - b^{\mathrm{L1O}}_{-i}$), as defined in Sec.~\ref{sec:method}.

\subsubsection{Action Space Size and Variance in Softmax Policy}
In the bandit setting with a softmax policy, the total variance of the policy gradient estimator is inherently tied to the dimensionality of the action space, denoted by $|\cA|$.
Here, we derive this relationship analytically.

Consider the score function for a selected action $a$, given by $\nabla_\theta \log \pi_\theta(a) = \mathbf{e}_a - \boldsymbol{\pi}_\theta$, where $\mathbf{e}_a$ is the one-hot vector for action $a$ and $\boldsymbol{\pi}_\theta$ is the probability vector.
Assuming the gradient variance is dominated by sampling noise (i.e., neglecting the squared norm of the expected gradient, $\|\E{\hat{g}}\|^2 \approx 0$), we approximate the total variance by the expected squared norm of the score function:
\begin{equation}
    V_{\text{total}}(\boldsymbol{\pi}_\theta) \approx \E[a \sim \pi_\theta]{ \left\| \nabla_\theta \log \pi_\theta(a) \right\|^2 }.
\end{equation}

Expanding the squared norm $\|\mathbf{e}_a - \boldsymbol{\pi}_\theta\|^2$ yields
\begin{align}
    \| \mathbf{e}_a - \boldsymbol{\pi}_\theta \|^2
    &= (\mathbf{e}_a - \boldsymbol{\pi}_\theta)^\top (\mathbf{e}_a - \boldsymbol{\pi}_\theta) \\
    &= \|\mathbf{e}_a\|^2 - 2 \mathbf{e}_a^\top \boldsymbol{\pi}_\theta + \|\boldsymbol{\pi}_\theta\|^2 \\
    &= 1 - 2 \pi_\theta(a) + \|\boldsymbol{\pi}_\theta\|^2.
\end{align}

Taking expectation over $a \sim \pi_\theta$ gives
\begin{align}
    V_{\text{total}}(\boldsymbol{\pi}_\theta)
    &= \sum_{a=1}^{|\cA|} \pi_\theta(a)\left( 1 - 2 \pi_\theta(a) + \|\boldsymbol{\pi}_\theta\|^2 \right) \\
    &= \sum_{a=1}^{|\cA|} \pi_\theta(a) - 2 \sum_{a=1}^{|\cA|} \pi_\theta(a)^2 + \|\boldsymbol{\pi}_\theta\|^2 \sum_{a=1}^{|\cA|} \pi_\theta(a) \\
    &= 1 - 2 \|\boldsymbol{\pi}_\theta\|^2 + \|\boldsymbol{\pi}_\theta\|^2 \\
    &= 1 - \|\boldsymbol{\pi}_\theta\|^2. \label{eq:softmax_variance}
\end{align}
Eq.~\eqref{eq:softmax_variance} shows that this approximation depends only on the squared $\ell_2$-norm of the probability vector.

To see the effect of the action space size $|\cA|$, consider the uniform policy (which is common at initialization), where $\pi_\theta(a)=1/|\cA|$ for all $a$.
Then
\begin{equation}
    V_{\text{total}}(\boldsymbol{\pi}_{\text{uniform}})
    = 1 - \sum_{a=1}^{|\cA|}\left(\frac{1}{|\cA|}\right)^2
    = 1 - |\cA|\cdot\frac{1}{|\cA|^2}
    = 1 - \frac{1}{|\cA|}.
\end{equation}

This implies that the variance increases with $|\cA|$ and approaches $1$ as $|\cA|\to\infty$.
For example, a binary bandit ($|\cA|=2$) yields $V_{\text{total}}(\boldsymbol{\pi}_{\text{uniform}})=0.5$, whereas $|\cA|=1000$ yields $V_{\text{total}}(\boldsymbol{\pi}_{\text{uniform}})=0.999$.
Thus, large action spaces can exhibit substantially higher gradient variance in the early stages of learning under a softmax policy.

\subsubsection{Additional Results}\label{app:toy_experiments:bandits:additional_results}
Here, we report additional results for the bandit setting.
In particular, we plot the variance of the gradient estimator for different batch sizes $B \in \{8, 16, 32\}$.

\begin{figure}[H]
    \centering
    \includegraphics[width=0.70\textwidth]{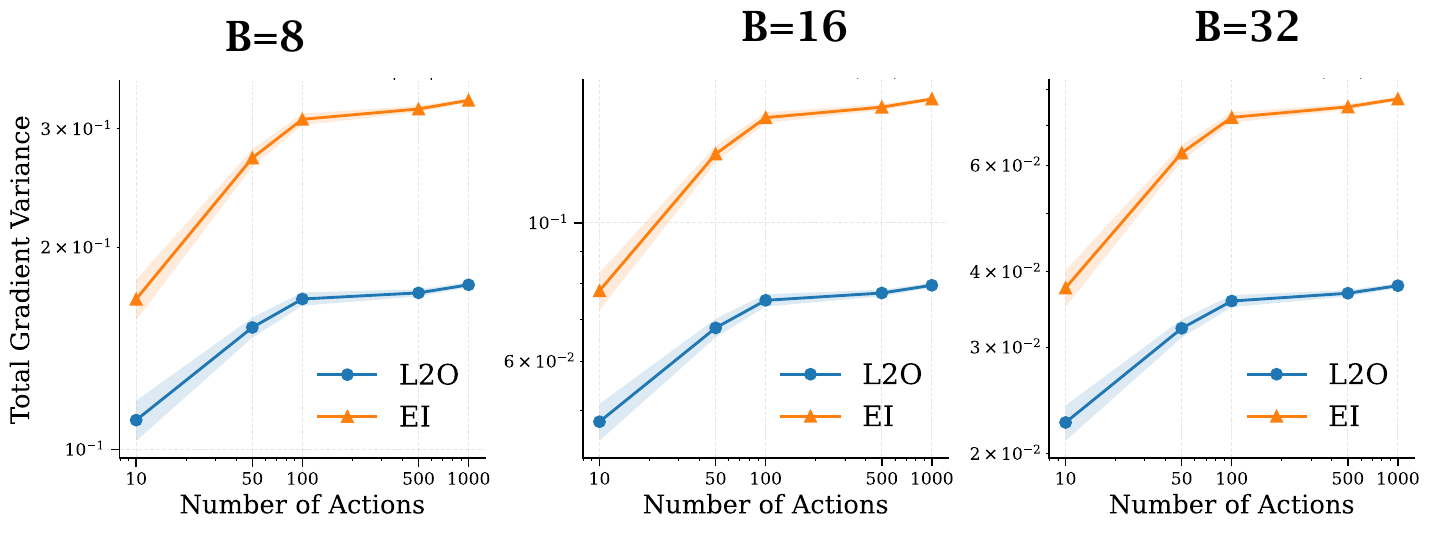}
    \caption{
        Variance vs action space size.
    }
    \label{app:toy_experiments:bandits:additional_results:variance_vs_action_space_size}
\end{figure}

\paragraph{Variance vs action space size.}
Here, we report additional results on variance versus action space size (Figure~\ref{app:toy_experiments:bandits:additional_results:variance_vs_action_space_size}).
We plot the variance of the gradient estimator for action space sizes $|\cA| \in \{10, 50, 100, 1000\}$, with batch sizes $B \in \{8, 16, 32\}$ and $K=2$.
As expected, the variance decreases as the batch size increases.
Moreover, the variance-reduction effect of the L2O baseline becomes more pronounced as the action space size increases, consistent with the results in Sec.~\ref{sec:toy_experiment:bandits}.

\begin{figure}[H]
    \centering
    \includegraphics[width=0.70\textwidth]{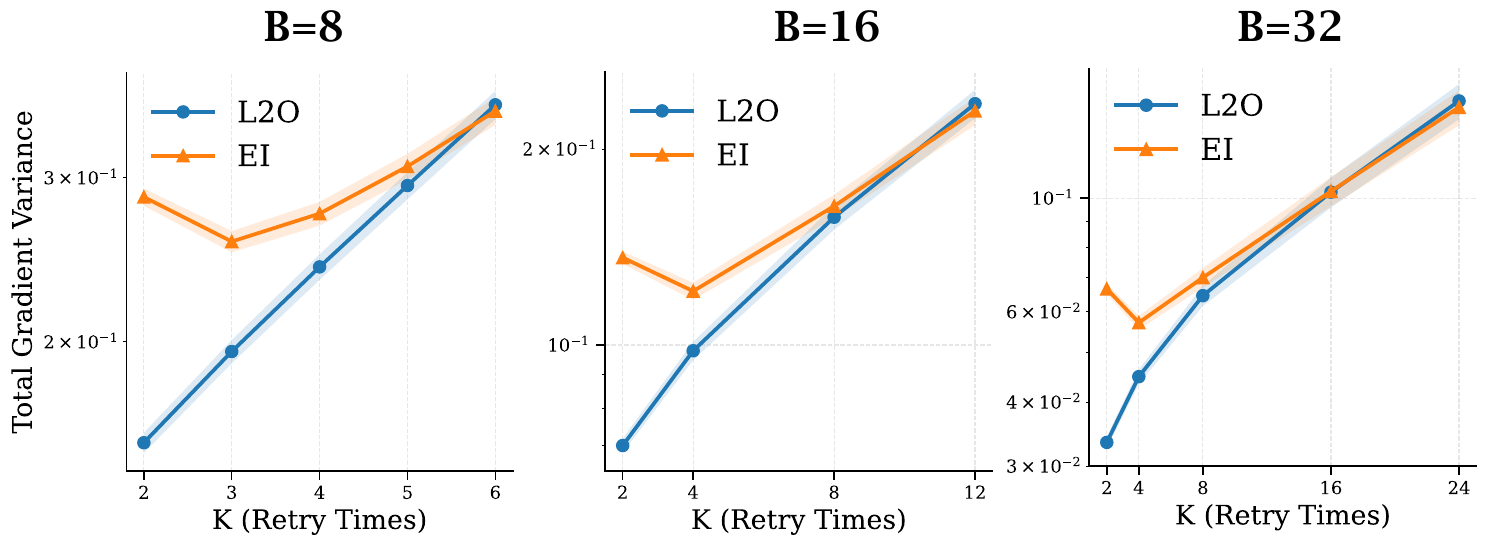}
    \caption{
        Variance vs $K$.
    }
    \label{app:toy_experiments:bandits:additional_results:variance_vs_m}
\end{figure}

\paragraph{Variance vs $K$.}\label{app:toy_experiments:bandits:additional_results:variance_vs_K}
Here, we report additional results on variance versus $K$ (Figure~\ref{app:toy_experiments:bandits:additional_results:variance_vs_m}).
We plot the variance while varying $K$, with batch sizes $B \in \{8, 16, 32\}$ and $|\cA|=100$.
As before, the variance decreases as the batch size increases.
We also observe that the variance-reduction effect of the L2O baseline is largest when $K$ is moderate relative to the batch size $B$, which is consistent with the results in Sec.~\ref{sec:toy_experiment:bandits}.

\subsection{Maze Environment}\label{app:toy_experiments:maze_environment}
Here, we report additional results for the maze environment.

\subsubsection{Setting}\label{app:toy_experiments:maze_environment:setting}

\paragraph{Environment.} 
\begin{wrapfigure}{r}{0.32\textwidth}
    \vspace{-1.0em}
    \centering
    \includegraphics[width=0.32\textwidth]{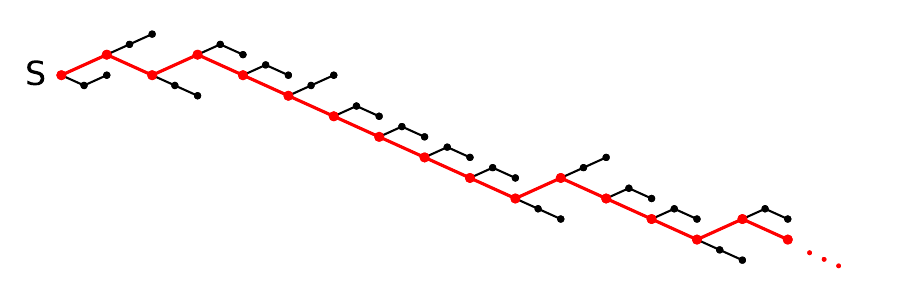}
    \caption{
        Structure of the Biased Maze Environment. The red line indicates the unique optimal path.
    }
    \label{fig:biased_maze}
    \vspace{-1.0em}
\end{wrapfigure}

The maze is deterministic, with binary actions ($0$ or $1$) available at each state.
The agent receives a reward of $+1$ for each forward step along a valid path.
There exists a single correct path to the goal (indicated by the red line in Fig.~\ref{fig:biased_maze}); if the agent selects the optimal action at every step, it can progress for up to 1000 steps.
If the agent takes an incorrect action, the episode terminates within one or two steps.
We design the maze to exhibit structural bias, where the correct action is $1$ in $75\%$ of the states.
Consequently, a policy that blindly prioritizes action $1$ can achieve moderately high returns, creating a local optimum that discourages exploration of the optimal path.

\paragraph{Policy Parameterization.}
We parameterize the policy using tabular logits for each state--action pair, augmented with a global bias vector, $\pi(a \mid s) = \mathrm{softmax}(\theta_{s,a} + \phi_a)$, where $\theta \in \R^{|S| \times |A|}$ and $\phi \in \R^{|A|}$.
The global parameter $\phi$ captures an environmental bias toward specific actions and is used as a proxy for exploratory tendency, as it represents the default policy in unvisited states (since $\theta$ is initialized near $0$).

\paragraph{Estimator Construction for Multi-Step RL.}\label{app:toy_experiments:maze_environment:estimator_construction_for_multi_step_rl}
In episodic RL settings, the optimization objective is defined over the cumulative return of an episode rather than immediate rewards.
Let $\tau_i = (s_{i,0}, a_{i,0}, r_{i,0}, \dots)$ denote the $i$-th trajectory in a batch of size $B$.
We define the episode return as $R_i = \sum_{t} r_{i,t}$ and the cumulative score function as $\Psi_i = \sum_{t} \nabla_\theta \log \pi_\theta(a_{i,t} | s_{i,t})$.
To apply the estimators derived in Section \ref{sec:method}, we simply substitute the immediate reward $r_i$ with the episode return $R_i$ and the score function $\psi(x, a_i)$ with $\Psi_i$.
The specific gradient estimators compared in our experiments are:

\begin{itemize}
    \item \textbf{Standard REINFORCE (with L1O Baseline):}
    Optimizes the standard expected return $J(\theta) = \E[x, a \sim \pi_\theta]{R(x, a)}$. We use the standard Leave-One-Out (L1O) baseline to reduce variance:
    \begin{equation}
        \hat{g}_{\mathrm{std}} := \frac{1}{B} \sum_{i=1}^{B} \Psi_i (R_i - b_{-i}), \quad \text{where } b_{-i} = \frac{1}{B-1} \sum_{j \neq i} R_j.
    \end{equation}
    For entropy-regularized REINFORCE, we added the entropy bonus term $\beta H(\pi_\theta)$ to the objective, where $H(\pi_\theta) = -\sum_{a} \pi_\theta(a) \log \pi_\theta(a)$ and $\beta$ is set to 0.01.

        \item \textbf{EI (Vanilla):}
    Optimizes the max@K objective using the raw Expected Improvement score $s_i$ (calculated via Eq.~\eqref{eq:pg_ei} in Sec.~\ref{sec:method:ei} using returns $\{R_j\}_{j=1}^B$):
        \begin{equation}
        \hat{g}_{\mathrm{ei}} := \frac{K}{B} \sum_{i=1}^{B} \Psi_i s_i.
    \end{equation}

    \item \textbf{EI + L2O (Ours):}
    Optimizes the max@K objective using the centered advantage $s_i - b_{-i}^{L2O}$:
        \begin{equation}
        \hat{g}_{\mathrm{l2o}} := \frac{K}{B} \sum_{i=1}^{B} \Psi_i (s_i - b_{-i}^{L2O}),
    \end{equation}
    where $b_{-i}^{L2O}$ is the L2O baseline computed from $\{R_j\}_{j=1}^B$ as defined in Theorem \ref{thm:l2o_baseline}.
\end{itemize}

\paragraph{Hyperparameters.}
The learning rate is set to $0.03$ after a grid search over $\{0.01, 0.03, 0.05\}$.
We collect trajectories with batch sizes $B \in \{5, 8, 16, 32\}$ and train for 3000 iterations.
We report the mean and standard error over $10$ different random seeds.

\paragraph{Metrics.}
To quantify variance reduction, we measure the total variance of the gradient estimator $\text{Tr}(\mathbb{V}[\hat{g}])$.
In Figure~\ref{fig:maze_return_all}, we report the "Grad Var Ratio", which is defined as the ratio of the gradient variance of our method (EI + L2O) to that of the vanilla EI estimator:
\begin{equation}
    \text{Ratio} = \frac{\sum_d \text{Var}(\hat{g}_{\mathrm{l2o}}^{(d)})}{\sum_d \text{Var}(\hat{g}_{\mathrm{ei}}^{(d)})},
\end{equation}
where the sum is taken over all parameters. A ratio less than $1.0$ indicates effective variance reduction.

\subsubsection{Results}\label{app:toy_experiments:maze_environment:additional_results}

\begin{figure}[H]
    \centering
    \includegraphics[width=0.90\textwidth]{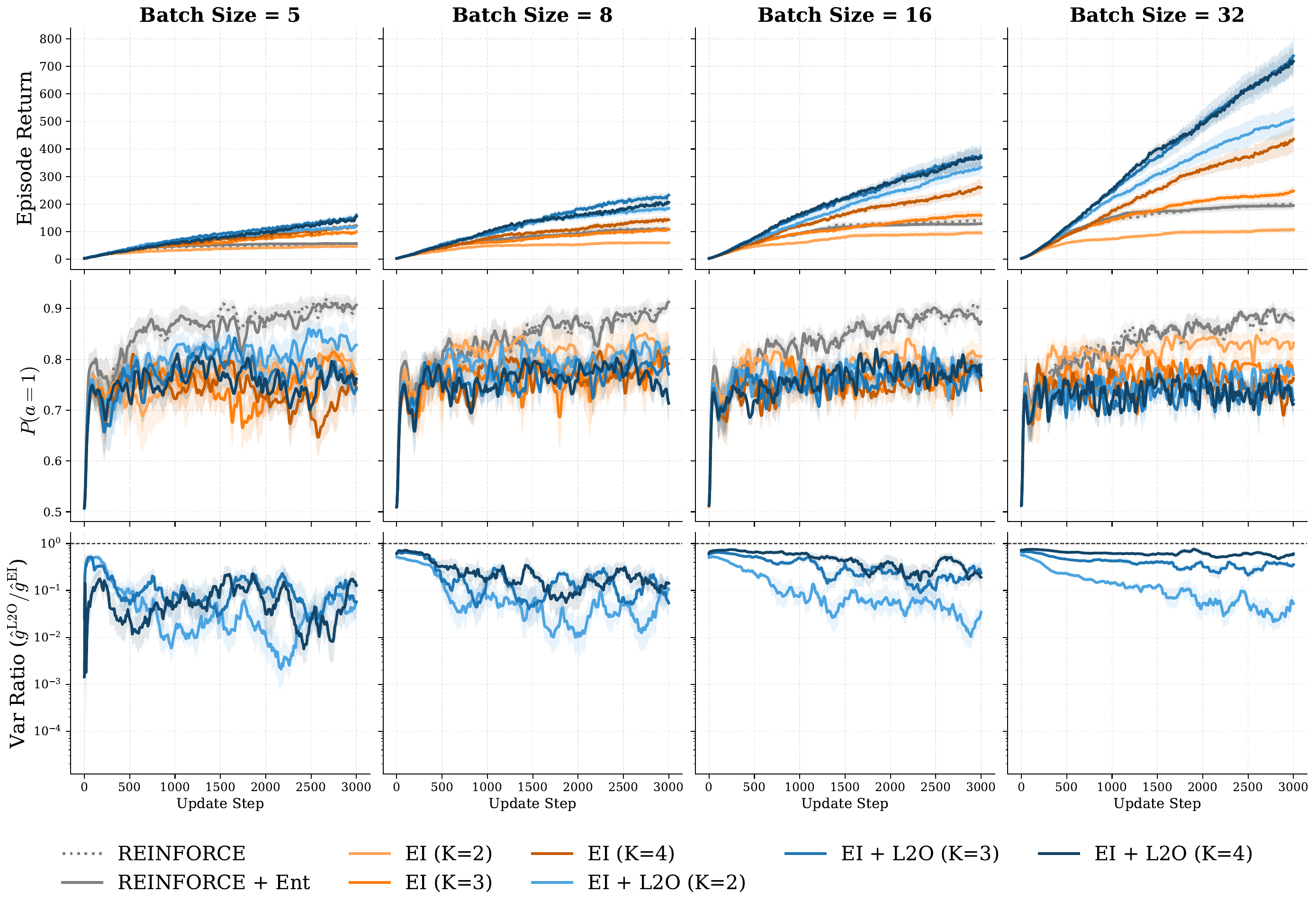}
    \caption{
        Moving average (window size 10) of the average return, probability of action 1 (via global bias $\phi_a$), and gradient variance ratio across different batch sizes $B \in \{5, 8, 16, 32\}$.
        The error bars are computed using the standard error of the moving average.
        Our method (EI + L2O) consistently outperforms standard REINFORCE, Entropy-regularized REINFORCE, and vanilla EI.
        The probability plots indicate that EI-based methods maintain an exploratory policy (prob. $\approx 0.75$) matching the environmental prior, whereas REINFORCE baselines converge to a greedy strategy (prob. $> 0.9$).
        The gradient variance ratio (EI + L2O / EI) remains consistently below 1.0, demonstrating the effectiveness of the L2O baseline.
    }
    \label{fig:maze_return_all}
\end{figure}

In Figure~\ref{fig:maze_return_all}, we present the full results for the maze environment across varying batch sizes $B \in \{5, 8, 16, 32\}$, reporting the average return, the probability of selecting action 1 (via the global bias parameter $\phi_a$), and the gradient variance ratio.
We observe that \textbf{EI + L2O consistently outperforms} all baselines—standard REINFORCE, Entropy-regularized REINFORCE, and vanilla EI—across all tested batch sizes.
The probability plots reveal that EI + L2O maintains a policy with a strong exploratory tendency (action 1 probability $\approx 0.75$), preventing the collapse to the suboptimal greedy strategy observed in the REINFORCE baselines.
Furthermore, the gradient variance ratio is consistently less than 1.0, confirming that the L2O baseline effectively reduces variance in all settings.
Notably, the variance-reduction effect appears more pronounced when the comparator set size is small (e.g., $K=2$).
This is expected, as increasing the number of comparators ($K$) tends to drive more Expected Improvement values to zero, since it becomes harder to beat the maximum of a larger set. As a result, the additional variance reduction available from L2O naturally decreases.
Nevertheless, the L2O baseline provides a consistent benefit across all configurations.

%% file: sec/app/llm_experiments.tex
\section{LLM Experiments}\label{app:llm_experiments}
In this section, we provide additional results for the LLM experiments presented in Sec.~\ref{sec:reasoning_experiments}.

\subsection{Setting}
\label{app:llm_experiments:setting}
All models are trained with a learning rate of $10^{-6}$, a batch size of 1024, and an optimization mini-batch
size of 256. For each input problem, we roll out 8 responses using a temperature of 1.0.
The Qwen model is trained on the MATH training split~\citep{hendrycks2021math}, while the Llama model is trained on a combined dataset of the GSM8K training split~\citep{cobbe2021training} and MATH Level 1 training examples.
We exclude evaluation examples from the RL training data, including all problems used in AIME24, AIME25, AMC23, MATH500, and Minerva evaluation.

\subsection{Additional Results}
\label{app:llm_experiments:additional_results}

Tables~\ref{tab:llm_passk_avg_256} and \ref{tab:llm_pass1_pass256_each} summarize the same evaluation protocol as the main text, reporting (i) task-average pass@k up to $k\le256$ and (ii) per-benchmark pass@1 and pass@256.
Figure~\ref{fig:llm_all} provides the per-benchmark pass@k curves to visualize how improvements distribute across datasets and across inference compute.
In addition, Figure~\ref{fig:llm_support_dynamics} analyzes how RL training changes the empirical support of correct solutions relative to the base model, following the taxonomy of \citet{wu2025invisible}.  
Specifically, for each problem under a matched sampling budget, we categorize outcomes into \textit{Support Preservation}, \textit{Support Shrinkage}, \textit{Support Expansion}, or \textit{Out of Support} (see the caption of Figure~\ref{fig:llm_support_dynamics}).

\begin{table}[t]
\centering
\setlength{\tabcolsep}{6pt}
\renewcommand{\arraystretch}{1.15}
\caption{\textbf{Task-average pass@k ($k\le256$).} Unweighted average over AIME24, AIME25, AMC23, MATH500, and Minerva under the evaluation protocol in Sec.~\ref{sec:reasoning_setup}.}
\label{tab:llm_passk_avg_256}
\begin{tabular}{lccccccccc}
\toprule
Method & 1 & 2 & 4 & 8 & 16 & 32 & 64 & 128 & 256 \\
\midrule
\textit{Qwen2.5-Math-7B}\\
\midrule
GRPO        & \textbf{40.7} & 45.5 & 49.2 & 52.2 & 55.4 & 58.7 & 62.1 & 65.6 & 69.3 \\
Entropy-Adv & 38.4 & 43.7 & 48.3 & 52.3 & 56.3 & 60.2 & 64.1 & 67.7 & 70.9 \\
PKPO        & 37.2 & 42.6 & 47.4 & 51.4 & 55.1 & 59.0 & 62.9 & 66.8 & 70.6 \\
MaxPO (Ours) & 39.9 & \textbf{46.5} & \textbf{51.7} & \textbf{56.2} & \textbf{60.1} & \textbf{63.8} & \textbf{67.5} & \textbf{71.1} & \textbf{74.3} \\
\midrule
\textit{Llama-3.2-3B-Instruct}\\
\midrule
GRPO        & \textbf{22.4} & \textbf{28.0} & 33.3 & 38.2 & 42.9 & 47.7 & 52.7 & 57.5 & 62.1 \\
Entropy-Adv & 22.1 & 27.7 & 33.2 & 38.5 & 43.6 & 48.5 & 53.3 & 58.2 & 63.1 \\
PKPO        & 20.1 & 25.5 & 31.3 & 37.1 & 42.5 & 47.5 & 52.7 & 58.0 & 63.5 \\
MaxPO (Ours) & 21.8 & 27.6 & \textbf{33.4} & \textbf{38.7} & \textbf{43.8} & \textbf{48.9} & \textbf{54.0} & \textbf{59.5} & \textbf{65.0} \\
\bottomrule
\end{tabular}
\end{table}

\begin{table}[t]
\centering
\setlength{\tabcolsep}{6pt}
\renewcommand{\arraystretch}{1.15}
\caption{\textbf{Per-benchmark pass@1 / pass@256.} Results for each benchmark under the same evaluation protocol as Sec.~\ref{sec:reasoning_setup}.}
\label{tab:llm_pass1_pass256_each}
\begin{tabular}{lcccccc}
\toprule
Method & AIME24 & AIME25 & AMC23 & MATH500 & Minerva & Avg. \\
\midrule
\multicolumn{7}{l}{\textit{Qwen2.5-Math-7B}}\\
\midrule
GRPO        & \textbf{30.3}/72.0 & 10.5/41.1 & \textbf{63.3}/96.1 & \textbf{76.3}/92.9 & \textbf{23.0}/44.5 & \textbf{40.7}/69.3 \\
Entropy-Adv & 25.6/72.3 & 9.8/47.3 & 59.5/97.0 & 74.7/93.6 & 22.4/44.2 & 38.4/70.9 \\
PKPO        & 19.9/72.9 & \textbf{11.1}/49.2 & 60.7/94.4 & 72.9/93.1 & 21.4/43.4 & 37.2/70.6 \\
MaxPO (Ours) & 28.5/\textbf{74.2} & 10.6/\textbf{54.7} & 62.2/\textbf{98.0} & 75.6/\textbf{95.7} & 22.9/\textbf{48.8} & 39.9/\textbf{74.3} \\
\midrule
\multicolumn{7}{l}{\textit{Llama-3.2-3B-Instruct}}\\
\midrule
GRPO        & \textbf{13.7}/50.5 & 0.5/34.3 & \textbf{30.2}/93.5 & 52.0/90.5 & \textbf{15.5}/41.6 & \textbf{22.4}/62.1 \\
Entropy-Adv & 12.6/49.2 & \textbf{0.9}/\textbf{36.5} & 29.5/94.8 & \textbf{52.4}/91.7 & 15.1/43.2 & 22.1/63.1 \\
PKPO        & 9.3/54.4 & 0.5/34.0 & 27.4/\textbf{96.8} & 49.2/90.7 & 14.2/41.4 & 20.1/63.5 \\
MaxPO (Ours) & 11.7/\textbf{58.1} & 0.6/35.4 & 29.9/96.0 & 51.7/\textbf{91.8} & 15.3/\textbf{43.8} & 21.8/\textbf{65.0} \\
\bottomrule
\end{tabular}
\end{table}

\clearpage

\begin{figure}[H]
    \centering
    \includegraphics[width=0.90\textwidth]{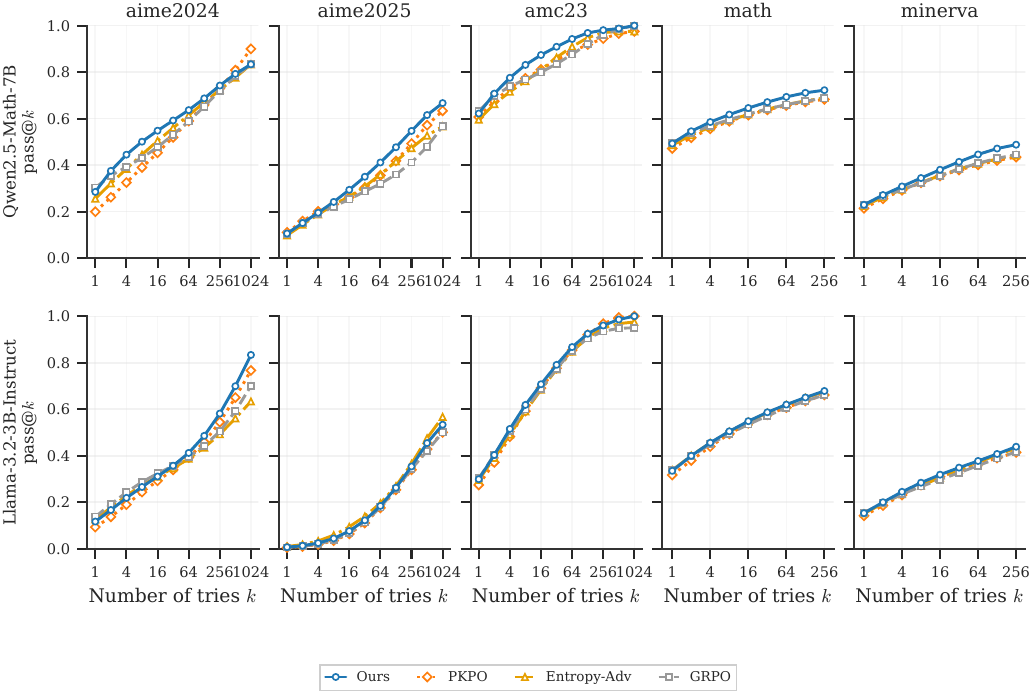}
        \caption{\textbf{Per-benchmark pass@k curves.} For each benchmark, we plot pass@k as a function of inference compute $k$ for all methods. This complements the task-average curves in Figure~\ref{fig:passk_avg_256} by showing where gains come from across datasets and inference-compute levels, and whether improvements persist as $k$ increases.}
    \label{fig:llm_all}
\end{figure}

\begin{figure}[H]
    \centering  \includegraphics[width=0.90\textwidth]{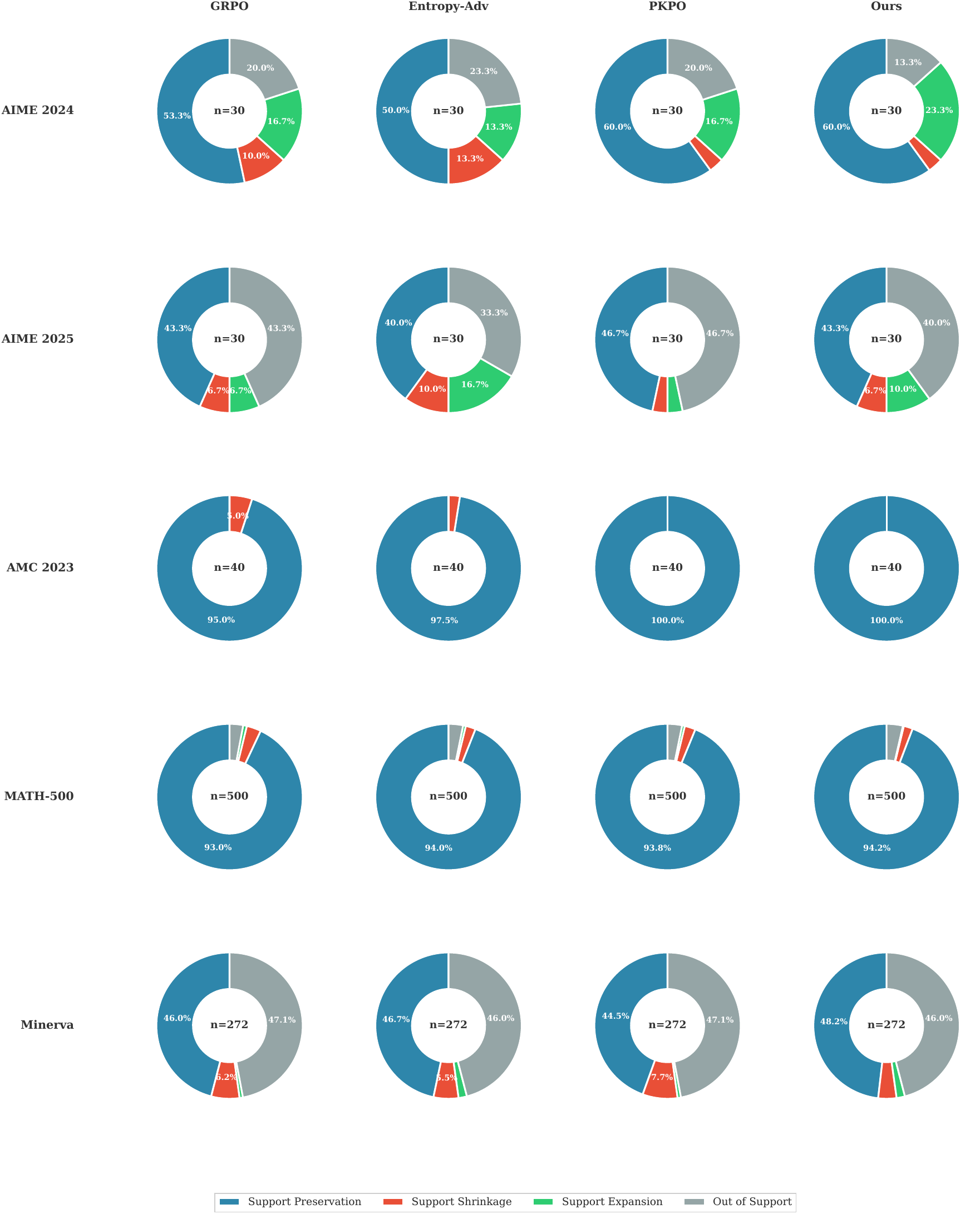}
        \caption{Empirical support dynamics relative to the base model, following the taxonomy of \citet{wu2025invisible}. For each benchmark, we categorize each problem into one of four cases under a matched sampling budget: (i) \textit{Support Preservation}: both the base model and the RL-trained model produce at least one correct completion; (ii) \textit{Support Shrinkage}: the base model succeeds but the RL-trained model fails; (iii) \textit{Support Expansion}: the RL-trained model succeeds but the base model fails; and (iv) \textit{Out of Support}: neither model produces a correct completion. Donut charts report the proportions for each RL method across benchmarks (sample sizes are shown in the plots).}
    \label{fig:llm_support_dynamics} 
\end{figure}

\subsubsection{Effect of Objective Size $K$}
\label{app:ablation_k}

We study the impact of the training objective size $K$ by comparing $K\in\{2,4\}$ while keeping other settings fixed.
As shown in Figure~\ref{fig:ablation_k}, increasing $K$ consistently decreases pass@1 but improves pass@k for large $k$ (e.g., $k=256$) for both Llama and Qwen.
This indicates an exploration-exploitation trade-off: increasing $K$ promotes broader exploration, which reduces one-shot accuracy but improves the probability of success when multiple samples are available.
Overall, $K=2$ offers the best balance in our setting and is used as the default.

\begin{figure}[t]
  \centering

  \begin{subfigure}[t]{0.49\linewidth}
    \centering
        \includegraphics[width=\linewidth]{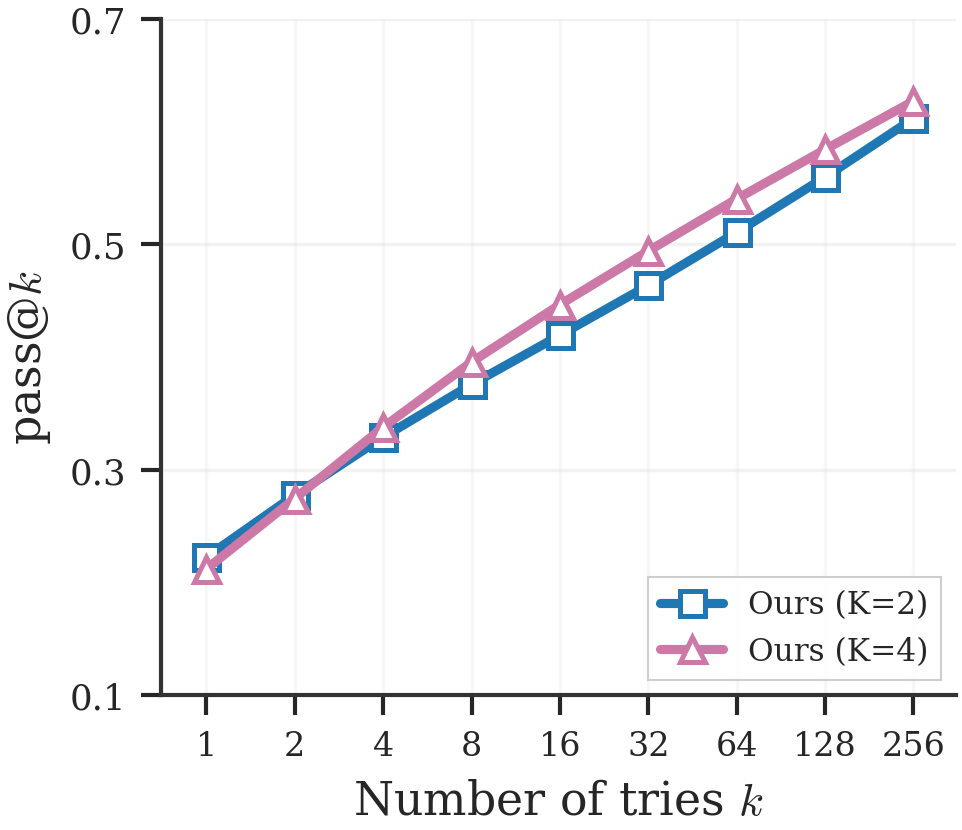}
                                \caption{Llama-3.2-3B-Instruct}
    \label{fig:ablation_k_llama}
  \end{subfigure}\hfill
  \begin{subfigure}[t]{0.49\linewidth}
    \centering
        \includegraphics[width=\linewidth]{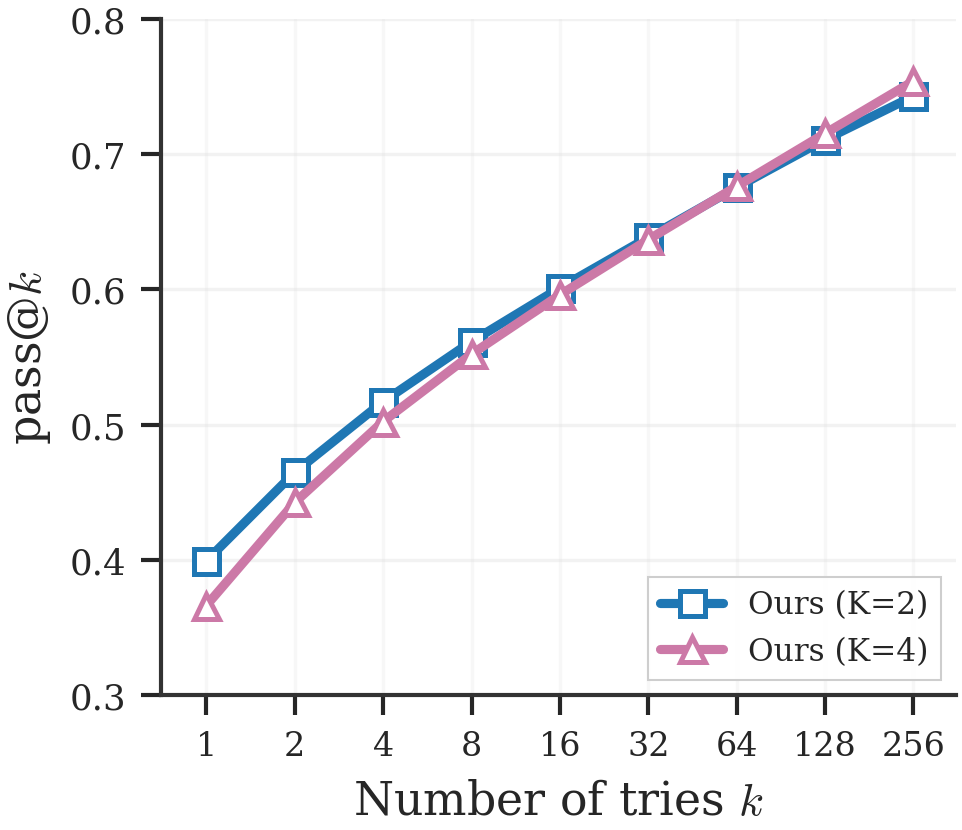}
                                \caption{Qwen2.5-Math-7B}
    \label{fig:ablation_k_qwen}
  \end{subfigure}

  \caption{Sensitivity to the training objective size $K$ ($K\in\{2,4\}$). Reported curves are task-average pass@k (unweighted average over AIME24, AIME25, AMC23, MATH500, and Minerva). For both Llama and Qwen, increasing $K$ lowers pass@1 while boosting pass@k for large $k$, indicating an exploration-exploitation trade-off: broader exploration reduces one-shot accuracy but improves success probability given more samples.}
  \label{fig:ablation_k}
\end{figure}